\documentclass[twoside,11pt]{article}

\usepackage{jmlr2e}

\usepackage{amsmath,amssymb,bm,bbm}

\usepackage{algorithmicx}
\usepackage{algorithm}
\usepackage[noend]{algpseudocode}
\usepackage{subcaption}
\usepackage{xr}
\usepackage{cleveref}

\usepackage{graphicx}

\usepackage{xcolor,soul}

\DeclareMathOperator*{\argmax}{arg\,max}

\newcommand{\cjqe}{\color{black}}

\newcommand{\NAM}{\textsc{DART}}
\usepackage{natbib}
\usepackage{hyperref}
\makeatletter
\def\@editor{}
\makeatother
\begin{document}
	
	%
	
	%
	
\title{DART:  aDaptive Accept RejecT for Non-Linear Top-$K$ Subset Identification}
\author{Mridul Agarwal, Vaneet Aggarwal, Christopher J. Quinn, and Abhishek K. Umrawal \thanks{M. Agarwal and A. K. Umrawal were with Purdue University, West Lafayette IN 47907, USA when this work was accomplished.  V. Aggarwal is with Purdue University, West Lafayette IN 47907, USA, email: vaneet@purdue.edu. C. J. Quinn is with Iowa State University, email: cjquinn@iastate.edu. This version is an extended version of AAAI 2021 paper \citep{agarwal2021dart}.} }
	\editor{}

\maketitle

\begin{abstract}
We consider the bandit problem of selecting $K$ out of $N$ arms at each time step. The reward can be a non-linear function of the rewards of the selected individual arms.
The direct use of a multi-armed bandit algorithm requires choosing among $\binom{N}{K}$ options, making the action space large. To simplify the problem, existing works on combinatorial bandits {typically} assume feedback as a linear function of individual rewards. In this paper, we prove the lower bound for top-$K$ subset selection with bandit feedback with possibly correlated rewards.
We present a novel algorithm for the combinatorial setting without using individual arm feedback or requiring linearity of the reward function. Additionally, our algorithm works on correlated rewards of individual arms.  Our algorithm, aDaptive Accept RejecT (\NAM), sequentially finds good arms and eliminates bad arms based on confidence bounds.
\NAM\  is computationally efficient and uses storage linear in $N$.
Further, \NAM\  achieves a regret bound of $\tilde{\mathcal{O}}(K\sqrt{KNT})$ for a time horizon $T$, which matches the lower bound in bandit feedback up to a factor of $K\sqrt{\log{2NT}}$. When applied to the problem of cross-selling optimization and maximizing the mean of individual rewards, the performance of the proposed algorithm surpasses that of state-of-the-art algorithms. We also show that \NAM\ significantly outperforms existing methods for both linear and non-linear joint reward environments.
\end{abstract}
\section{Introduction}
We consider the problem of finding the best subset of $K$ out of $N$ items to optimize a possibly non-linear function of the reward of each item. We note that the joint reward as a function of individual rewards is a much more natural setting to understand and arises in a number of settings.
For example, in 
the problem of erasure-coded storage \citep{xiang2016joint}, the agent chooses $K$ out of $N$ servers to obtain the content for each request; the final reward is the negative of the time taken by the slowest server. A recommendation system agent may present a list of $K$ items out of $N$ items to the user for a non-zero reward only if the user selects an item \citep{kveton2015cascading} from the list.  Similarly, in cross-selling item selection, a retailer creates a bundle with $K$ items, and the joint reward is a quadratic function of the selected items' individual rewards \citep{1250942}. The problem of a daily advertising campaign is characterized by a set of sub-campaigns where the aggregate reward is the sum of the rewards of sub-campaigns \citep{zhang2012joint,nuara2018combinatorial}.  Combinatorial Multi-Armed Bandit (CMAB) algorithms can solve these problems in an online manner.  For many CMAB algorithms, we can bound the regret, that is the loss incurred from accidentally selecting  sub-optimal sets some of the time.  
We aim to find a space and time-efficient CMAB algorithm that minimizes cumulative regret.

{

Existing algorithms for $K=1$ that use Upper Confidence Bound (UCB) or Bayesian resampling methods
\citep{auer02using,auer2002finite,auer2010ucb,thompson1933likelihood,agrawal2012analysis,gopalan14thompson} can bound the regret by $\tilde{\mathcal{O}}(\sqrt{NT})$.
These methods can be naturally extended to the combinatorial setting where $K$ arms are chosen, treating each of the $\binom{N}{K}$ possible actions as a distinct `arm'.  Unfortunately, this approach has two significant drawbacks. First, the regret increases exponentially in $K$ as the number of total actions to explore has grown from $N$ to $\binom{N}{K}$. Second, the time and space complexities  increase exponentially in $K$, requiring storage of values for all actions to find the  action with the highest UCB \citep{auer2010ucb,auer2002finite} or highest sampled rewards \citep{agrawal2012analysis}.
}



This paper addresses those issues by proposing a novel algorithm called \textbf{aDaptive Accept RejecT} (\NAM).    To estimate the ``goodness'' of an arm, we use the mean of the rewards obtained by playing actions containing that arm.  In an adaptive manner, \NAM\  moves arms to ``accept'' or ``reject'' sets based on those estimates, reducing the number of arms that require further exploration.
{{We assume that the expected joint reward of an arm $i$ with any choice of the remaining $K-1$ arms is better than the expected joint reward of arm $j$ with the same $K-1$ other arms if arm $i$ is individually better than arm $j$. }} This assumption is naturally satisfied in many online decision-making settings, such as  click-model bandits \citep{kveton2015cascading,lattimore2018toprank}.
We then use Lipschitz continuity of the joint reward function to relate orderings between pairs of arms $i$ and $j$  to orderings between pairs of actions containing those arms. 
We construct a martingale sequence to analyze the regret bound of \NAM. Furthermore, \NAM\ achieves a space complexity of $\mathcal{O}(N)$ and a per-round time complexity of $\tilde{\mathcal{O}}(N)$.  

The main contributions of this paper can be summarized as follows:
\begin{enumerate}
    \item We propose \NAM\ - a time and space efficient algorithm for the non-linear top-$K$ subset selection problem with only the joint reward as feedback.   We show that \NAM\ has a per-step time complexity of $\Tilde{\mathcal{O}}(N)$ and space complexity of $\mathcal{O}(N)$.
    
    \item We prove a lower bound of {$\Omega(\sqrt{NKT})$} for the regret of top-$K$ subset selection problem for a linear setup where the joint reward is the mean of individual rewards and the individual rewards are possibly correlated.
    
    \item We prove that \NAM\ achieves a (pseudo-)regret of $\Tilde{\mathcal{O}}(K\sqrt{NKT})$ over a time horizon $T$ and under certain assumptions.
\end{enumerate}

We also empirically evaluate the proposed algorithm \NAM, comparing it to other, state-of-the-art full-bandit feedback CMAB algorithms.  We first consider a linear setting, where the joint reward is simply the mean of individual arm rewards.  We also examine the setting where the joint reward is a  quadratic function of individual arm rewards, based on the problem of cross-selling item selection \citep{1250942}.  Our algorithm significantly outperforms existing state-of-the-art algorithms, while only using polynomial space and time complexity.  



\section{Related Works} \citet{dani08stochastic,NIPS2007_3371,Cesa-Bianchi:2012:CB:2240304.2240495,Audibert:2014:ROC:2765232.2765234,abbasi-yadkori11improved,li2010contextual,agrawal2013thompson} consider a linear bandit setup where at time $t$, the agent selects a  length $N$ vector $x_t$ from the decision set $D_t\subset\mathbb{R}^N$ and observes a reward $\theta^Tx_t$ for an unknown constant vector $\theta\in\mathbb{R}^N$. The algorithms proposed in these works use the linearity of the reward function to estimate rewards of individual arms and achieve a regret of $\tilde{\mathcal{O}}(\sqrt{NT})$.
\citet{filippi10parametric,jun17scalable,li17provably} studied the problem of generalized linear models (GLM) where the reward $r_t$ is a function $\left(f(z):\mathbb{R}\to\mathbb{R}\right)$ of $z=\theta^T x_t$ plus some noise. Generalized linear model algorithms also obtain a regret bound of $\tilde{\mathcal{O}}(\sqrt{NT})$. 
{{These (generalized) linear model algorithms can naturally be extended to our setup for linear joint reward functions.}} However, the space and time complexity remains exponential in $K$ to store all possible $\binom{N}{K}$ actions.


For the setting of $K=1$, \citet{pmlr-v84-liau18a} reduces the space complexity from $\mathcal{O}(N)$ to $\mathcal{O}(1)$ at the cost of worse regret bounds. When {extended to the combinatorial setting, treating each set of $K$ arms as a distinct `arm,'} the regret bound becomes exponential in $K$. Recently, \citet{rejwan2020top} bounded the regret by $\mathcal{O}(K{\sqrt{NT}})$ for identifying the best $K$ subset, in the case when the joint reward is the sum of rewards of independent arms, using $\mathcal{O}(N)$ space and {per-round} time complexity. \citet{pmlr-v32-lind14} considered the combinatorial bandit problem with a non-linear reward function and additional feedback, where the feedback is a linear combination of the rewards of the $K$ arms. Such feedback 
allows for the recovery of individual rewards. \citet{agarwal2022stochastic} proposed a divide-and-conquer-based algorithm for the best $K$ subset problem with a non-linear joint reward with bandit feedback.  
\citet{agarwal2022stochastic} worked on a setup that is most similar to ours. 
Algorithms by \citet{pmlr-v32-lind14,agarwal2022stochastic} achieve $\mathcal{O}(T^{2/3})$ regret while the proposed algorithm in this paper achieves $\mathcal{O}(T^{1/2})$ regret. 

Many works have studied the semi-bandit setting, where individual arm rewards are also available as feedback \citep{kveton2014matroid,pmlr-v28-chen13a,kveton2014matroid,lattimore2018toprank,gai2012combinatorial,gai2010learning}. 
\citet{kveton2014matroid} provides a UCB-type algorithm for matroid bandits, where the agent selects a maximal independent set of rank $K$ to maximize the sum of individual arm rewards. \citet{pmlr-v28-chen13a} considered the combinatorial semi-bandit problem with non-linear rewards using a UCB-type analysis. {\color{black}Recently, \citet{merlis2019batch} studied the setup of combinatorial semi-bandits with non-linear rewards. The joint reward is a linear combination of multiple possible reward functions of combinatorial actions. Further, each reward function is possibly non-linear and Gini-smooth. With access to an oracle, they obtain an upper bound of  $\tilde{\mathcal{O}}(\gamma_g\sqrt{NT})$ on the regret where $\gamma_g$ is Gini-smoothness parameter which scales with $K$. Further, they also show that the gap-dependent bound of their algorithm is optimal up to logarithmic factors \citep{merlis2020tight}.}
In contrast to these prior works, we consider the full-bandit setting where individual arm rewards are not available. \citet{kveton2015tight,lattimore2018toprank} proved a lower bound of $\Omega(\sqrt{NKT})$ for semi-bandit problems {where the joint reward is simply the sum of individual arm rewards.} \citet{kalyanakrishnan2012pac} also provides a lower bound for the best $K$ subset problem of $\Omega\left(\frac{N}{\epsilon^2}\log{\left(\frac{K}{\delta}\right)}\right)$ for any $(\epsilon,\delta)$-PAC algorithm playing single arm at each time. \citep{audibert2014regret} obtained a lower bound of $\Omega(K\sqrt{NT})$ for bandit feedback and provide an algorithm with regret bound of $\Omega(K\sqrt{NKT})$ for linear bandits without assuming independence between arms. \citet{cohen2017tight} obtain a tighter lower bound of $\Omega(K\sqrt{KNT})$ for a bandit setup where the rewards of individual arms are possibly correlated. { We obtain a regret bound for $\Omega(\sqrt{KNT})$ for a problem setup where the total joint reward lies between [0,1]. For such a setup, our proposed algorithm achieves a lower bound which is tight up to a factor of $K$ and logarithmic terms for bandit feedback with possibly correlated rewards.}

{
\Cref{tab:lit-review} provides the comparison of the difference in scale in terms of regret bound and per-step time complexity between our and related works.

\begin{table}[h]
\caption{Comparison of different algorithms.} \label{tab:lit-review}
\centering
\begin{tabular}{|c|c|c|c|}
    \hline \hline
    Algorithm & Setup & Regret Bound & Per-Step Time Complexity\\
    \hline
    DART (This paper) & Non-Linear&  $\Tilde{O}(K\sqrt{NT})$&  $\Tilde{O}(N)$ \\
    UCB \citep{auer2010ucb} &  General& $\Tilde{O}(\sqrt{{N\choose K} T})$ &  $\Tilde{O}({N\choose K})$ \\
    CSAR \citep{rejwan2020top} &  Linear &  $\Tilde{O}(K\sqrt{NT})$ & $\Tilde{O}(K^2)$ \\  
    CMAB-SM \citep{agarwal2022stochastic} & Non-Linear&$\Tilde{O}(K^{1/2}N^{1/3}T^{2/3})$ & $\Tilde{O}(K)$\\
    LinTS \citep{agrawal2013thompson} & Linear&  $\Tilde{O}(N\sqrt{T})$& $\Tilde{O}(N^3)$ \\
    LinUCB \citep{li2010contextual} & Linear & $\Tilde{O}(\sqrt{NT})$& $\Tilde{O}(N^3)$\\
    \hline
    \hline
\end{tabular}
\end{table}
}
\section{Problem Formulation}\label{sec:formulation}
We consider $N$ ``arms'' labeled as $i \in [N] = \{1, 2, \cdots, N\}$. On playing arm $i$ at time step $t$, it generates a reward $X_{i,t} \in [0,1]$ which is a random variable. We assume that $X_{i,t}$ are independent across time, and for any arm the distribution 
 is identical at all times. {For simplicity, we will use $X_i$ instead of $X_{i,t}$ for analysis that holds for any $t$.} 
{The distribution for each arm $i$'s rewards $\{X_{i,t}\}_{t=1}^T$ could be discrete, continuous, or mixed.}

The agent can only play an action $\bm{a}\in\mathcal{N}$ where $\mathcal{N} = \{{\bm a}\in[N]^K \ \big| \  {\bm a}(i) \neq {\bm a}(j)\ \forall \ i,j:\ 1 \leq i< j\leq K\}$ is the set of all $K$ sized tuples created using arms in $[N]$. Thus, the cardinality of $\mathcal{N}$ is $\binom{N}{K}$. For an action $\bm{a}$, let $\bm{X}_{\bm{a},t} = (X_{\bm{a}(1),t}, X_{\bm{a}(2),t}, \cdots ,X_{\bm{a}(K),t})$ be  the column reward vector of individual arm rewards at time $t$ from arms in action $\bm{a}$. The reward $r_{\bm{a},t}$ of an action $\bm{a}$ at time $t$ is a bounded function $f:[0,1]^K\to[0,1]$ of the individual arm rewards
\begin{align}
r_{\bm{a},t} &= f({\bm X}_{{\bm a}, t}). \label{eq:R_at_t}
\end{align}
{As $X_{i,t}$ are i.i.d. across time $t$, $\bm{X}_{\bm{a},t}$ are also i.i.d. across time $t$ for all $\bm{a}\in\mathcal{N}$. Later in the text, we will skip index $t$, for brevity, where it is unambiguous.} We denote the expected reward of any action ${\bm a} \in \mathcal{N}$ as  $\mu_{\bm a} = \mathbb{E}[r_{\bm a}]$.
We assume that there is a unique ``optimal'' action ${\bm a}^*$ for which the expected reward  $\mu_{ {\bm a}^*}$ is highest among all actions,
\begin{align}
{\bm a}^* &= \argmax_{{\bm a}\in\mathcal{N}}\ \mu_{\bm a}.
\end{align} 

At the time $t$, the agent plays an action $\bm{a}_t$ randomly sampled from an arbitrary distribution over $\mathcal{N}$ dependent on the history of played actions and observed rewards till time $t-1$.
The agent aims to reduce {the} cumulative (pseudo-)regret $R$ over time horizon $T$, defined as the expected difference between the rewards of the best action in hindsight and the actions selected by the agent.
\begin{align}
    R &= \mathbb{E}_{\bm{a}_1, r_{\bm{a}_1}(1), \cdots, \bm{a}_T, r_{\bm{a}_T}(T)}\left[\sum_{t=1}^T r_{\bm{a}^*}(t) - r_{\bm{a}_t}(t)\right] \\
    &= T\mu_{\bm{a}^*}-\mathbb{E}_{\bm{a}_1, \cdots, \bm{a}_T, }\left[\sum_{t=1}^T  \mu_{\bm{a}_t}\right]. \label{eq:regret}
\end{align}

 %
%

We define   the gap $\Delta_{i,j}$ between two arms $i$ and $j$ as the difference between the expected rewards of arm $i$ and arm $j$,
\begin{align}
    \Delta_{i,j} = \mathbb{E}\left[X_i\right] -\mathbb{E}\left[X_j\right].
\end{align}

We now mention the assumptions for this paper. We first assume that the joint reward function $f$ is permutation invariant. 
Let $\Pi$ denote the set of all permutation functions of a length $K$ vector.  
\begin{assumption}[Symmetry] \label{symmetry_assumption}
    For all ${\bm a} \in \mathcal{N}$ and for any permutation $\pi \in \Pi$ of the vector ${\bm X}_{\bm a}$ of individual arm rewards, 
    \begin{align}
        f\left({\bm X}_{\bm a}\right) = f\left(\pi({\bm X}_{\bm a})\right)
    \end{align}
\end{assumption}
{\color{black}
Using Assumption \ref{symmetry_assumption}, we note that $f(\cdot)$ is essentially a function of the set of individual arm rewards of arms in $\bm{a}\in\mathcal{N}$. Hence, we define $\bm{X}_{\bm{S}_1\cup\bm{S}_2}$, where $\bm{S}_1\cup\bm{S}_2\in\mathcal{N}$ as a vector of rewards of individual arms in $\bm{S}_1\cup\bm{S}_2\in\mathcal{N}$.
}




We also assume that the expected reward of the action with a good arm is higher than the expected reward of action with a bad arm for all possible combinations of the remaining $K-1$ arms. Further, if two arms, $i$, and $j$, are equally good, then without loss of generality the resulting actions with arm $i$ will be at least as good as the resulting actions with arm $j$ for all possible combinations of the remaining $K-1$ arms.

{

\begin{assumption}[Good arms generate good actions] \label{monotone_assumption}
    We assume that if and only if the expected reward of arm $i$ is higher than the expected reward of arm $j$ (for any given $i\ne j$), then for any subset $\bm{S}$ of size $K-1$ arms chosen from the remaining $N-2$ arms (arms excluding $i$ and $j$), the expected reward of $\bm{S}\cup \{i\}\in\mathcal{N}$ is higher than the expected reward of $\bm{S}\cup \{j\}\in\mathcal{N}$.  More precisely, 
    \begin{align}
        \mathbb{E}[X_i] > \mathbb{E}[X_j]& \iff \mathbb{E}\left[f\left( \bm{X}_{\bm{S}\cup \{i\}}\right)\right] > \mathbb{E}\left[f\left( \bm{X}_{\bm{S}\cup \{j\}}\right)\right] \text{, and}\nonumber\\
        \mathbb{E}[X_i] = \mathbb{E}[X_j]& \iff \mathbb{E}\left[f\left( \bm{X}_{\bm{S}\cup \{i\}}\right)\right] = \mathbb{E}\left[f\left( \bm{X}_{\bm{S}\cup \{j\}}\right)\right]\label{eq:monotone_assumption}
    \end{align}
for all $\bm{S}$ and $i,j\in[N]$.
\end{assumption}
}




We also assume that $f(\cdot)$ is bi-Lipschitz continuous (in an expected sense).

{

Let $\mathbb{E}[{\bm{X}}_{{\bf a}}]$ denote the vector of mean arm rewards for arms in ${\bf a}$. Then we have the following assumption.
\begin{assumption}[Continuity of expected rewards] \label{continuous_assumption}
    The expected value of $f(\cdot)$ is bi-Lipschitz continuous with respect to the expected value of the rewards obtained by the individual arms,
    if there exists a $U< \infty$ such that, 
    \begin{align}
        \frac{1}{U} \min_{\pi'\in\Pi}{\big\|}\mathbb{E}[{\bm{X}}_{{\bf a}_1}] - \pi'(\mathbb{E}[{\bm{X}}_{{\bf a}_2}]){\big\|_1} \leq \big|\mu_{{\bf a}_1} -\ \mu_{{\bf a}_2} \big|
        \leq \ U {\big\|}\mathbb{E}[{\bm{X}}_{{\bf a}_1}] - \pi(\mathbb{E}[{\bm{X}}_{{\bf a}_2}]){\big\|_1} \label{eq:cont_lower_defn_n}
    \end{align}   
    for any pair of actions ${\bf a}_1,{\bf a}_2 \in \mathcal{N}$ and for any permutation $\pi$ of $\bm{X}$. \footnote{We note that followed by our AAAI paper, (Wagde and Saha, OPT 2025) pointed an issue in the earlier version, where there was no min in the left hand side. Having the minimum fixes the issue. }
\end{assumption}
}

    

{Using the continuity assumption, we obtain the following corollary which bounds the expected reward of individual of any two arms $i,j$ using an action by replacing arms $i$ with arm $j$. We have:
\begin{corollary} \label{inverse_continuity}
    For any arms $i, j\in\mathcal{N}$ and any subset $\bm{S}\subset \mathcal{N}\setminus \{i,j\}$ of size $K-1$, 
        \begin{align}
    \big|\mathbb{E}[X_i] - \mathbb{E}[X_j]\big|  
    \leq U\big| \mu_{\bm{S}\cup\{i\}} - \mu_{\bm{S}\cup\{j\}}\big|.
    \end{align}
\end{corollary}
\begin{proof}
We obtain the result by choosing $\bm{a}_1 = \bm{S}\cup\{i\}$ and $\bm{a}_2 = \bm{S}\cup\{j\}$.
\end{proof}
}


{

Assumptions 1-3 are satisfied for many problem setups, such as in the  cascade model for clicks \citep{kveton2015cascading} where a user interacting with a list of documents clicks on the first documents the user likes. The joint reward  $r_{\bm{a}_t, t} = \max (X_{\bm{a}_t(1),t},\cdots,X_{\bm{a}_t(1),t})$ is the maximum of individual arm rewards where, for all $i\in[N], t\in[T]$, $X_{i,t}$ follows Bernoulli distribution with $\mathbb{E}[X_{i,t}]\in(0,1)$. Since $\max(a,b) = \max(b,a)$, the Assumption \ref{symmetry_assumption} is satisfied. For Assumption \ref{monotone_assumption}, consider a set, $\bm{S}$, of size $K-1$ by selecting arms from $[N]\setminus\{i,j\}$. Then actions $S\cup\{i\}$ and $S\cup\{i\}$ have expected rewards $1-\left(\Pi_{k\in\bm{S}}(1-\mathbb{E}[X_k])\right)(1-\mathbb{E}[X_i]) > 1-\left(\Pi_{k\in\bm{S}}(1-\mathbb{E}[X_k])\right)(1-\mathbb{E}[X_j])$ which holds when $\mathbb{E}[X_i]>\mathbb{E}[X_j]$ Further, the bi-Lipschitz property in Assumption \ref{continuous_assumption} in individual expected rewards holds from the fact that $\mathbb{E}[\mu_{\bm{S}\cup\{i\}}] - \mathbb{E}[\mu_{\bm{S}\cup\{i\}}] = \left(\Pi_{k\in\bm{S}}(1-p_k)\right)(p_i- p_j)$.

} 

The assumptions are also satisfied in cross-selling optimization \citep{1250942}, where the reward is a quadratic function of the individual items sold in a bundle $K$ with $r_{\bm{a},t} = \bm{X}_{\bm{a},t}^T\bm{A}\bm{X}_{\bm{a},t}$ for some matrix $\bm{A}$. The assumptions are also satisfied for joint reward functions such as the mean of individual arm  rewards \citep{rejwan2020top,pmlr-v28-chen13a}.


\section{Lower bound on top-$K$ subset identification}
Given the problem formulation, we now prove a tight lower bound on the subset identification problem for a linear joint reward function with correlated arms. We consider a specific setup.  {Let $\bm{a}^* = (1, 2, \cdots, K)$ denote the tuple of best arms which is initially unknown to the agent}. Define the reward function as {\color{black}$f(\bm{X}) = (1/K)\sum_{i=1}^K\bm{X}(i)$}, where $\bm{X}(i)$ is the $i^{th}$ entry of $\bm{X}$. The individual arm distributions are of the form $X_{i,t}' = 1/2 + \epsilon\bm{1}_{\{i\in\bm{a}^*\}} + Z_t$, where $Z_{t}$ follows a Gaussian distribution with mean $0$ and variance $\sigma^2$ and $\bm{1}_{\{\cdot\}}$ denotes the indicator function.  The arms are correlated through the shared additive term $Z_t$.  
\cjqe

{\color{black}
\begin{theorem} \label{thm:lower_bound}
For $\epsilon = \frac{\sigma}{2}\sqrt{\frac{NK}{2T}}$, any deterministic player must suffer expected regret of at least $\Omega(\sigma \sqrt{KNT})$ against an environment with rewards $X_{i,t}'$ for $t = 1, 2, \cdots, T$ for each arm $i\in \mathcal{N}$.
\end{theorem}
\begin{proof}(Outline:) The proof is based on the proof techniques presented in \citep{cohen2017tight,audibert2014regret}. We note that if the algorithm plays against a setup where all the arms are identically distributed, then the expected number of times it selects an arm $i\in\bm{a}^*$ is $KT/N$ as the arms are not distinguishable. Using this and the proof of Lemma 4 from  \citep{cohen2017tight} we obtain the required result. A detailed proof is provided in Appendix \ref{app:lower_bound}.
\end{proof}

Note that for the lower bound, we considered a general setup with $X_{i,t}'\in(-\infty, \infty)$. However, our setup bounds individual rewards in $[0, 1]$. This can again be managed by the proof technique from \citep[Theorem 5]{cohen2017tight}, by bounding the probability of $X_{i,t}'>1$ for all $t\leq T$ by choosing $\sigma^2 = 1/(4\log NKT)$.
}




\section{Proposed \NAM\ Algorithm}\label{sec:algo_describe}
The \NAM\ algorithm works by partitioning a uniformly random permutation of the arms. The algorithm initializes $\hat{\mu}_i$ as the estimated mean for actions that contain arm $i$ and $n_i$ as the number of times  an action containing arm $i$ is played. The algorithm proceeds in epochs, indexed by $e$, and maintains three different sets at each epoch. The first set, $\mathcal{A}_e$, contains ``good'' arms which belong to the top-$K$ arms found till epoch $e$. The second set, $\mathcal{N}_e$, contains the arms which the algorithm is still exploring at epoch $e$. The third set, $\mathcal{R}_e$, contains the arms that are ``rejected'' and do not belong in the top-$K$ arms. We let $K_e$ be the variable that contains the number of spots to fill in the top-$K$ subset at epoch $e$. The algorithm maintains a decision variable
$\Delta$ as the concentration bound and a parameter variable $n$ as the minimum number of samples required for achieving the concentration bound $\Delta$. Lastly, the algorithm maintains a hyperparameter $\lambda$ tuned for the value of $T, N$, and $K$. $\lambda$ is the minimum gap between any two arms the algorithm can resolve within time horizon $T$.

In line  5, the algorithm selects a permutation of $\mathcal{N}_e$  uniformly at random and partitions it into sets of size $K_e$ {\color{black}using Algorithm \ref{alg:partition_algo}}. If $K_e$ does not divide $|\mathcal{N}_e|$, we repeat arms in the last group (cyclically; so that the last group has $K_e$ distinct arms). To simplify the bookkeeping, $\hat{\mu}_i$ and $n_i$ are not updated if arm $i$ is repeated in the last group. The algorithm then creates an action $\bm{a}_t$ from the partitioned groups and the arms in the good set and plays it to obtain a reward $r_{\bm{a}_t}(t)$ at time $t$ (lines 8-9). \NAM\ then updates the estimated mean for all arms played in $\bm{a}_t$ with the observed reward and increments the number of counts for the arms played (lines 10-12). 

In lines 15-16, 
the algorithm moves an arm $i\in\mathcal{N}_e$ to $\mathcal{A}_e$ if the estimated mean of actions that contain arm $i$, $\hat{\mu}_i$, is $\Delta$ more than the estimated mean of actions that contain arm at $(K+1)^{th}$ rank, $\hat{\mu}_{K+1}$. Similarly, the algorithm moves an arm $i\in\mathcal{N}_e$ to $\mathcal{R}_e$ if the estimated mean of actions that contain arm $i$, $\hat{\mu}_i$, is $\Delta$ less than the estimated mean of actions that contain arm at $K^{th}$ rank, $\hat{\mu}_{K}$.



\begin{algorithm}[t]
    \caption{\NAM($T$, $N$, $K$)}\label{alg:nsras}
	\begin{algorithmic}[1]
        \State Initialize $\hat{\mu}_i = 0, n_i = 0$ for $i \in \{1, 2, \cdots, N\};\ t = 1;\ e = 0,\lambda = \sqrt{\frac{720NK\log{2NT}}{T}}$ 
		\State $\mathcal{A}_e = \phi, \mathcal{R}_e = \phi, \mathcal{N}_e = [N]$ \Comment{Initialize parameters for rounds }
		\State $\Delta = 1, n = \frac{32\log(NT)}{\Delta^2},K_e = K-|\mathcal{A}_e|$
        \While{$t <= T$}
            \State {\color{black}$\mathcal{N}_{e,1}, \mathcal{N}_{e,2}, \cdots, \mathcal{N}_{e,\lceil|\mathcal{N}_e|/K_e\rceil}$ = \textsc{Partition Arms}($\mathcal{N}_e, K_e$)}
            \State $e = e+1; \ell = 1$
            \While{$\ell \leq \lceil |\mathcal{N}_e|/K_e\rceil$ and $t < T$}
                \State $\bm{a}_t = \mathcal{A}_e \cup \mathcal{N}_{e,\ell}$\Comment{Create action from arms in $\mathcal{A}_e \cup \mathcal{N}_{e,\ell}$}
                \State Play action $\bm{a}_t$ and obtain reward $r_{\bm{a}_t,t}$
                \ForAll{ arm $i\in \mathcal{N}_{e, \ell}$}
                    \State $\hat{\mu}_i = \frac{n_i\hat{\mu}_i + r_{\bm{a}_t}(t)}{n_i+1}; n_i = n_i + 1$
                \EndFor
                \State $t=t+1; \ell =\ell + 1$
            \EndWhile
            
            \If{$e \geq n$}
                \State Sort arms in $\mathcal{N}_e$ according to $\hat{\mu}_{(1)} \geq \hat{\mu}_{(2)} \geq\cdots\geq \hat{\mu}_{(|\mathcal{N}_e|)}$
                \State $\Bar{\mathcal{A}} = \{i\in\mathcal{N}_e|\hat{\mu}_{(i)}\ge\hat{\mu}_{(K + 1 - |\mathcal{A}_e|)} + \Delta\};\ \Bar{\mathcal{R}} = \{i\in\mathcal{N}_e|\hat{\mu}_{(i)}\le\hat{\mu}_{(K - |\mathcal{A}_e|)} - \Delta\}$
                \State $\mathcal{A}_{e+1} = \mathcal{A}_e\cup\Bar{\mathcal{A}};\ \mathcal{R}_{e+1} = \mathcal{R}_e\cup\Bar{\mathcal{R}}$
                \State $\mathcal{N}_{e+1} = \mathcal{N}_e\setminus\left(\Bar{\mathcal{A}}\cup\Bar{\mathcal{R}}\right);\ K_{e+1} = K-|\mathcal{A}_{e+1}|$
                \State $\Delta = \frac{\Delta}{2},  n = \frac{32\log(NT)}{\Delta^2}$
            \Else
                \State $\mathcal{N}_{e+1} = \mathcal{N}_e; \mathcal{A}_{e+1} = \mathcal{A}_e; \mathcal{R}_{e+1} = \mathcal{R}_e$ 
            \EndIf
            \color{black}
            \If{$\Delta < \lambda$ or $|\mathcal{A}_e\cup\mathcal{N}_e| == K$}
                \State break \textbf{while} loop
            \EndIf
        \EndWhile
        \State {\color{black}Sort $\mathcal{N}_e$ according to $\hat{\mu}_{(1)} \geq \hat{\mu}_{(2)} \geq\cdots\geq \hat{\mu}_{(|\mathcal{N}    _e|)}$}
        \State {\color{black}$\bm{a} = \mathcal{A}_e\cup\{(1),\cdots,(K - |\mathcal{A}_e|) \}$} \Comment{Append top $K-|\mathcal{A}_e|$ arms of $\mathcal{N}_e$ to create action $\bm{a}$}
        \While{$t <= T$}
            \State Play action $\bm{a}; t= t+1$
        \EndWhile
	\end{algorithmic}
\end{algorithm}

\begin{algorithm}[t]
    {\color{black}
    \caption{\textsc{Partition Arms}($\mathcal{N}_e, K_e$)}\label{alg:partition_algo}
	\begin{algorithmic}[1]
	    \State Select a permutation of $\mathcal{N}_e$, $i_1, \cdots, i_{|\mathcal{N}_e|}$, uniformly at random
        \For{$\ell \in \{1, \cdots, \lceil|\mathcal{N}_e|/K_e\rceil\}$}
            \State $\mathcal{N}_{e,\ell} = \left\{\left(K_e(\ell -1)~\texttt{mod} |\mathcal{N}_e|\right) + 1,\cdots,\left( (K_e\ell - 1) ~\texttt{mod} |\mathcal{N}_e|\right) + 1\right\}$
        \EndFor
		\Return $\mathcal{N}_{e,1}, \mathcal{N}_{e,2}, \cdots, \mathcal{N}_{e,\lceil|\mathcal{N}_e|/K_e\rceil}$
	\end{algorithmic}}
\end{algorithm}

The proposed \NAM\ algorithm uses a random permutation of $\mathcal{N}_e$. The random permutation can be generated in $\mathcal{O}(N)$ steps. Also after each round, the algorithm finds the $K^{th}$ and $(K+1)^{th}$ ranked arms. This operation can be completed in $\Tilde{
\mathcal{O}}(N)$ time complexity \textcolor{black}{by} sorting $\{\hat{\mu}_i\}_{i=1}^N$. Also going over each arm in $\mathcal{N}_e$ is of line ar time complexity. Hence, the per-step time complexity of the algorithm comes out to be $\Tilde{\mathcal{O}}(N)$. Also, the proposed \NAM\ algorithm only stores the estimates $\hat{\mu}_i$ for each arm $i\in[N]$. The resulting storage complexity is $\mathcal{O}(N)$ for maintaining the estimates. To find the top-$K$ and the top-$(K+1)$ means, the algorithm may use additional space of $\mathcal{O}(N)$ to maintain a heap. Thus, the overall space complexity of the algorithm is only $\mathcal{O}(N)$.




 

\section{Regret Analysis}
We now analyze the sample complexity and regret of the proposed \NAM\ algorithm. To bound the regret, we first bound the number of samples required to move an arm in $\mathcal{N}_e$ to either of $\mathcal{A}_e$ or $\mathcal{R}_e$. Then, we bound the regret from including a sub-optimal arm in the played actions. {\color{black}For the analysis, without loss of generality, we assume that the expected rewards of arms are ranked as $\mathbb{E}[X_1]\ge \mathbb{E}[X_2]\ge\cdots\ge\mathbb{E}[X_N]$. If the arms are not in the said order, we relabel the arms to obtain the required order.
From Assumption \ref{monotone_assumption}, we have $\bm{a}^* = \{1, 2, \cdots, K\}$ as an optimal arm. We refer to arms $1, \cdots, K$ as optimal arms and arms $K+1, \cdots, N$ as sub-optimal arms.} 
\subsection{Number of samples to move an arm in $\mathcal{N}_e$ to either of $\mathcal{A}_e$ or $\mathcal{R}_e$}\label{sec:num_samples}
We call two arms $i,j\in\mathcal{N}_e,\ i<j$ separated if the algorithm has high confidence that $\mathbb{E}[X_i]>\mathbb{E}[X_j]$. We first analyze the general conditions to separate any two arms $i,j\in\mathcal{N}_e$ such that $\mathbb{E}[X_i] > \mathbb{E}[X_j]$. 
Let the epoch where arm $i$ and arm $j$ are separated and the epoch of Algorithm \ref{alg:nsras} be $e$.
We define a filtration $\mathcal{F}_e$ as the history observed by the algorithm till epoch $e$.

{
For any $u\in [N]$, let $\mathcal{N}_e(u) = \{\bm{a}\in {\mathbb [N]}^K: u\in \bm{a}, \mathcal{A}_e(i) \in \bm{a} \forall i \in {1, \cdots, |\mathcal{A}_e|}, \mathcal{R}_e(i) \notin \bm{a} \forall i \in {1, \cdots, |\mathcal{R}_e|}, {\bm a}(i) \neq {\bm a}(j) \forall i,j: 1 \leq i< j\leq K\}$. This set $\mathcal{N}_e(u)$ is the set of all the actions which can be generated at epoch $e$ such that they contain arms $u$.
}
 

 We now define a random variable $Z_{i,j}(e)$ for $i, j \in \mathcal{N}_e$, which denotes the difference between the reward observed from playing a uniform random action  from $\mathcal{N}_e(i)$ and a uniform random action  from $\mathcal{N}_e(j)$. In other words, 
{
\begin{align}
    Z_{i,j}(e) = \begin{cases}
    r_{\bm{a}_i, t_{i,e}} - r_{\bm{a}_j, t_{j,e}}, & \text{for } i,j \in \mathcal{N}_e \\
    0, & \text{otherwise}
  \end{cases}
\end{align}

where $\bm{a}_i\sim \mathbb{U}(\mathcal{N}_e(i))$, $\bm{a}_j\sim \mathbb{U}(\mathcal{N}_e(j))$ and $\mathbb{U}(\cdot)$ denotes the uniform distribution.  
Also, $t_{i,e}$ is the time step in epoch $e$ at which the agent plays action $\bm{a}_i$ and obtains reward $r_{\bm{a}_i}(e)$. Similarly, $t_{e,j}$ is the time step in epoch $e$ at which the agent plays action $\bm{a}_j$ and obtains reward $r_{\bm{a}_j}(e)$.} Hence, the randomness of $Z_{i,j}(e)$ comes from {both} the random selection of $\bm{a}_i$ and $\bm{a}_j$, and from the reward generated by playing $\bm{a}_i$ and $\bm{a}_j$. Let $\mathbb{P}_{Z_{i,j}(e)}$ denote the probability distribution of $Z_{i,j}(e)$. We now mention a lemma for bounding the expected value of $Z_{i,j}(e)$ for all epochs $e$.

\begin{lemma}\label{lem:expected_Z_lemma}
Let $i, j\in\mathcal{N}_e$ be two arms such that $\mathbb{E}[X_i]>\mathbb{E}[X_j]$. Let $Z_{i,j}(e)$ be a random variable denoting the difference between the reward obtained on playing {a uniform random} action $\bm{a}_i{ \sim\mathbb{U}(\mathcal{N}_e(i))}$ containing arm $i$ and { a randomly selected} action $\bm{a}_j{ \sim\mathbb{U}(\mathcal{N}_e(j))}$ containing arm $j$. Then the expected value of $Z_{i,j}(e)$ is upper bounded by $U\Delta_{i,j}$, and lower bounded by $0$, or, 
\begin{align}
    \frac{\Delta_{i,j}}{UK}\leq \mathbb{E}[Z_{i,j}(e)|i,j \in \mathcal{N}_e] \leq U\Delta_{i,j}    
\end{align}
\end{lemma}
\begin{proof}[Proof Sketch]
The upper bound is obtained by calculating the number of possible actions $\bm{a}_1$ that contain arm $i$ and the number of possible actions $\bm{a}_2$ that contains arm $j$ and then applying the upper bound on $|\mu_i-\mu_j|$ from Assumption \ref{continuous_assumption}. Similarly, we obtain the lower bound by replacing the upper bound by the lower bound on $|\mu_i-\mu_j|$ from Assumption \ref{continuous_assumption}.
\end{proof}

The sequence of random variables $Z_{i,j}(e), e= 1,2,\cdots$ are not independent as the sets {$\mathcal{N}_e(i)$ and $\mathcal{N}_e(j)$} are updated as the algorithm proceeds. Hence, we cannot apply Hoeffding's concentration inequality \cite[Theorem 2]{hoeffding1994probability} for analysis. To use  Azuma-Hoeffding's inequality \citep[Chapter 3]{bercu2015concentration} (given as Lemma \ref{lem:azuma_hoeffding} in Appendix for completeness), we need to construct a martingale. {For each pair of arms $i,j\in \mathcal{N}_e$ with $\mathrm{E}[X_i]>\mathrm{E}[X_j]$,} we define {$Y_{i,j}$ as a martingale} with respect to filtration $\mathcal{F}_e$,
\begin{align}
    Y_{i,j}(e) = \sum_{e'=1}^e\left(Z_{i,j}(e') - \mathbb{E}_{e'-1}\left[Z_{i,j}(e')|\bm{1}_{\{i,j\in \mathcal{N}_{e'}\}}\right]\right)\label{eq:martingale_def}
\end{align}
where $\mathbb{E}_{e'-1}[Z_{i,j}(e')] = \mathbb{E}[Z_{i,j}(e')|\mathcal{F}_{e'-1}, i\in \mathcal{N}_{e'} ,j\in \mathcal{N}_{e'}]$. $Y_{i,j}(e)$ is a martingale with zero-mean, and $|Y_{i,j}(e)-Y_{i,j}(e-1)|\leq2$, and hence we can apply Azuma-Hoeffding's inequality to $Y_{i,j}(e)$ for all $i, j\in\mathcal{N}_e$. 


After obtaining the concentration of $Y_{i,j}(e)$ with respect to the $e^{th}$ iteration of sample of action with arm $i$ and arm $j$, we now obtain the value of $e$ for which we can consider arm $i$ and $j$ to be separated with probability $1-1/NT$. 

\begin{lemma}\label{lem:num_epochs_lemma}
Let arms $i, j\in\mathcal{N}_e$ be two arms such that $\mathbb{E}[X_i]>\mathbb{E}[X_j]$. Let  $\Delta$ be such that $\Delta \le \hat{\mu}_i -\hat{\mu}_j < 2\Delta$. Then, with probability at least $1-1/NT$, \NAM\ algorithm separates arm $i$ and arm $j$ with $\frac{32\log(2NT)}{(2\Delta)^2}\le e < \frac{32\log(2NT)}{\Delta^2} < \frac{288U^2K^2\log{2NT}}{\Delta_{i,j}^2}$.
\end{lemma}
{
\begin{proof}[Proof Sketch:]The \NAM\ algorithm separates two arms when $\hat{\mu}_i - \hat{\mu}_j\ge \Delta$. Also, for any $\Delta < 1$, the \NAM\ algorithm run with $n = \frac{32\log(2NT)}{(2\Delta)^2}$. Further, at epoch $e$, we have $\hat{\mu}_i-\hat{\mu}_j = \sum_{e'=1}^eZ_{i,j}(e')/ e.$
Using this relation and Azuma-Hoeffding's inequality on $Y_{i,j}(e)$, we get the required result. A detailed proof is provided in Appendix \ref{sec:proof_num_epochs_lemma}.
\end{proof}
}

{ We can now bound the number of samples required to move each arm from  $\mathcal{N}_e$ to either the ``accept'' set $\mathcal{A}_e$ or the ``reject'' set $\mathcal{R}_e$.  In the algorithm, arm $i$ will be moved to the accept set $\mathcal{A}_e$ when its empirical mean $\hat{\mu}_i$ is sufficiently larger than that of the $K+1$ ranked arm. Consider an arm $i $ in the optimal action $a^*=\{1, \dots, K\}$. {\color{black}The \NAM\ algorithm chooses $\Delta \in\{1, 1/2, 1/4, \cdots,\}$ and hence there will be a $\Delta$ such that $\Delta\le \Delta_{i,j}< 2\Delta$. } By Lemma \ref{lem:num_epochs_lemma}, with probability $1-1/NT$, arms $i$ and $K+1$ will be separable by epoch
\begin{align}
    e\leq \frac{288U^2K^2\log{(2NT)}}{\Delta_{i,K+1}^2}.
\end{align} 

Similarly, arm $i$ will be moved to the reject set $\mathcal{R}_e$ when its empirical mean $\hat{\mu}_i$ is sufficiently less than that of the $K$th ranked arm.  Consider an arm $i\in \{K+1, \dots, N\}$.  By Lemma \ref{lem:num_epochs_lemma}, with probability $1-1/NT$, arms $i$ and $K$ will be separable by epoch
\begin{align}
    \frac{288U^2K^2\log{(2NT)}}{\Delta_{K,i}^2}. \label{eq:epoch:separsubopt}
\end{align}

}

\subsection{Regret from sampling sub-optimal arms}
We first bound the regret of playing any action $\bm{a}\in\mathcal{N}$ 
using  Assumption \ref{continuous_assumption}. 

\begin{lemma}\label{lem:action_gap_permutation}
Let $\bm{a}=(a_1, a_2, \cdots, a_K)$ be any action.  The expected regret suffered from playing action $\bm{a}$ instead of action $\bm{a}^* = (1, 2, \cdots, K)$ is bounded as
\begin{align}
   \big|\ \mu_{\bf a} - \mu_{{\bf a}^*}\big| \leq \ U \sum_{i=1}^K\big|\mathbb{E}[X_{a_i}] - \mathbb{E}[X_{\pi(i)}]\big|,\label{eq:arm_single_instance_regret_bound}
\end{align}
for any permutation $\pi$ of $\{1, \cdots, K\}$ for which $\pi(i) = a_i$ if $a_i\leq K$.
\end{lemma} 
\begin{proof}[Proof Sketch:]
From Assumption \ref{continuous_assumption}, we first find a tight upper bound. We finish the proof by using the fact that Assumption \ref{continuous_assumption} selects the permutation which minimizes the bound, hence any other permutation also gives a valid upper bound. A detailed proof is provided in Appendix \ref{sec:proof_action_gap_permute}.
\end{proof}

We now bound the regret incurred by playing an action $\bm{a_t}$ at time $t$ containing  sub-optimal arm $i\in\{K+1, \cdots, N\}$ replacing an optimal arm $j\in\{1, \cdots,K\}$ in Lemma \ref{lem:arm_regret_bound}. 

\begin{lemma}\label{lem:arm_regret_bound}
For any sub-optimal action, the regret, $R_{K_i}$, it can accumulate by replacing an optimal arm $j\in\{1,\cdots, K\}$ by an arm $i\in{K+1,\cdots, N}$ is bounded by
\begin{align}
    R_{K,i} \le \frac{1440U^3K^2\log{(2NT)}}{\Delta_{K,i}}
\end{align}
\end{lemma}
\begin{proof}[Proof Sketch]
The agent suffers from regret if it an action that contains at least one sub-optimal arm. To bound the regret from a sub-optimal action, we use the proof technique of \cite{rejwan2020top} to divide the optimal arms $j\in\{1, \cdots, K\}$ into two groups: first group with $\Delta_{j,K+1}> \Delta_{K,i}$ and second group with $\Delta_{j,K+1}\leq \Delta_{K,i}$. { Now, we use the fact that the total number of epochs for which a sub-optimal arm $i$ will be played is upper bounded by   Equation \ref{eq:epoch:separsubopt}. Also, we use the fact that the expected regret the arm $i$ incurs by replacing any optimal arm $j$ is upper bounded by Equation \ref{eq:arm_single_instance_regret_bound}. With this,  We show that regret from both groups is bounded by $\mathcal{O}\left(\frac{1}{\Delta_{K,i}}\right)$.}
The detailed proof is provided in Appendix \ref{sec:proof_arm_regret_bound}.
\end{proof}

After calculating the regret from individual arms, we now calculate the total regret of the \NAM\ algorithm in the following theorem

{\color{black}
\begin{theorem}\label{thm:main_theorem}
The distribution-dependent regret incurred by \NAM\ algorithm is bounded by 
\begin{align}
    R \leq \mathcal{O}\left(\sum_{i:\Delta_{K,i}\ge \lambda}\frac{1440U^3K^2\log(2NT)}{\Delta_{K,i}} + 2UKT\lambda\right).
\end{align}
Choosing $\lambda = U\sqrt{\frac{720NK\log{2NKT}}{T}}$ bounds the distribution dependent regret incurred by \NAM\ algorithm by 
\begin{align}
    R \leq \mathcal{O}\left(U^2K\sqrt{NKT\log{2NT}}\right)
\end{align}
\end{theorem}
}
\begin{proof}[Proof Sketch]
We use the standard proof technique of bounding regret accumulated while eliminating arms to reject the set of a confidence bound-based algorithm to tune $\lambda$ and calculate the regret. The detailed proof is provided in Appendix \ref{sec:proof_main_theorem}.
\end{proof}

{
\begin{remark}We note that the regret bound of \NAM\ is looser compared to the lower bound in Theorem 2 by a factor of $K\log{(2NT)}$ for bandits with joint reward as the mean of rewards of individual arms in an action. This is because, the \NAM\ algorithm generates random actions using randomly sampled arms from the set $\mathcal{N}_e$ and the arms from accept set $\mathcal{A}_e$ at every epoch $e$. Now, to compare two arms, $i$ and $j$, they must belong to different actions otherwise we associate same reward to both the arms. The probability of the two arms assigned into a same action is $\frac{K_e-1}{|\mathcal{N}_e|-1}$. For a case where $|\mathcal{N}_e| = K_e + 1$, this probability is $\frac{K_e-1}{K_e}$, and we require to sample $K_e\le K$ more times to generate a meaningful comparison between the $K_e + 1$ arms. Because of this increased sampling requirement by a factor of $K$, our regret bound is looser by a factor of $K$.
\end{remark}
}

We note that there may be scenarios where an agent does not know the value of $U$ and cannot tune $\lambda$ accordingly. In such a case, the agent increases its regret because of not knowing the joint reward function. For a value of $\lambda = \sqrt{\frac{720NK\log{2NT}}{T}}$, which does not use $U$, the regret of the algorithm is bounded as 
\begin{align}
\mathcal{O}\left(\left(U^3+U\right)K\sqrt{NKT\log{2NT}}\right).    
\end{align}

Additionally, we note that we can convert \NAM\ to an anytime algorithm using the doubling trick of restarting the algorithm at $T_l = 2^l\ \forall\ l=1, 2, \cdots$ until the unknown time horizon $T$ is reached \citep{auer2010ucb}. Using analysis from \cite[Theorem 4]{besson2018doubling}, we show that \NAM\ for unknown $T$ achieves a regret bound of $\tilde{\mathcal{O}}(\sqrt{T})$. We present the complete proof  in Appendix \ref{sec:anytime_algorithm_proof}.

We also considered a case where the reward function is indeed linear, but a scaled function of individual rewards such that Equation \eqref{eq:cont_lower_defn_n} in the bi-Lipschitz Assumption \ref{continuous_assumption} modifies to
\begin{align}
    \big|\mu_{{\bf a}_1} -\ \mu_{{\bf a}_2} \big|
        = \ U {\big\|}\mathbb{E}[{\bm X}_{{\bf a}_1}] - \pi(\mathbb{E}[{\bm X}_{{\bf a}_2}]){\big\|_1} \label{eq:linear_function}
\end{align}
for some permutation $\pi$ of $\bm{X}$. Then, the same analysis holds and for $\lambda = \frac{1}{U}\sqrt{\frac{720NK\log2NT}{T}}$ the regret bound becomes independent of $U$ as:
\begin{align}
    R\leq \mathcal{O}\left(K\sqrt{NKT\log{2NT}}\right).
\end{align}
We noted that the regret becoming independent of $U$ is not surprising. The intuition behind this observation is, if $U$ becomes large then it is easy to separate arms and the regret does not grow large as the algorithm quickly finds the good arms. On the other hand, if $U$ becomes small, the cost of choosing a wrong arm is reduced by a factor of $U$.
{
\section{Experiments}\label{sec:experimental_results}
We experimentally evaluated the performance of the proposed algorithm. We discuss the reward setups, list baseline algorithms chosen for comparison, provide experimental details, and discuss the results.

\subsection{Reward Setups}
For our experiments, we focused on the following two joint reward functions. 
\begin{enumerate}
    \item A joint reward function that is a linear function of individual rewards, i.e., $r_{\bm{a}_t,t} = \bm{\theta}^T\bm{X}_{\bm{a}_t}$, where $\bm{\theta}\in\mathbb{R}^K$ is a vector with all entries as $1/K$. Such a linear reward function is used in slate selection for e-commerce \citep{dimakopoulou2019marginal, liu2021map}.
    \item A joint reward function is a quadratic function of individual rewards, i.e., $r_{\bm{a}_t,t} = \bm{X}_{\bm{a}_t}^T\bm{A}\bm{X}_{\bm{a}_t}$, where $\bm{A}\in\mathbb{R}^{K\times K}$ is an upper triangular matrix with all entries as $2/K(K+1)$. Such a quadratic reward function is used in cross-selling optimization to quantify the total profit from selling a bundle of items compared to the profit from selling the items in the bundle separately \citep{1250942}.
\end{enumerate}
For each joint reward function, each individual arm reward follows a Bernoulli distribution with mean sampled from $\mathbb{U}([0,1])$.

\subsection{Baseline Algorithms}
We chose the following baseline algorithms.
\begin{enumerate}
    \item UCB \citep{auer2010ucb} is a classic multi-armed bandit algorithm. It selects arms using the upper confidence bound of the reward estimates. Because of the combinatorial explosion of actions possible, we run it only for $K = 2$ for our setup.
    \item CSAR \citep{rejwan2020top} is the state-of-the-art algorithm for the case when the joint reward is the sum of rewards of independent arms. It involves an efficient sampling scheme that uses Hadamard matrices to accurately estimate the individual arms’ expected rewards.
    \item CMAB-SM \citep{agarwal2022stochastic} is the state-of-the-art algorithm for non-linear joint reward function with full-bandit feedback. It is a divide-and-conquer strategy that divides the original problem into several subproblems and efficiently combines the results of those individual subproblems. 
    \item LinTS (Scaled) \citep{agrawal2013thompson} is the state-of-the-art algorithm for linear joint reward function with full-bandit feedback. It uses the linearity of the reward function to estimate the rewards of individual arms. Further, it creates a confidence ellipsoid around the reward estimate and samples a new reward estimate for each round. The confidence ellipsoid scales with $O(N)$.
    \item LinTS (Unscaled) is the unscaled version of LinTS (Scaled). We remove the $O(N)$ scaling of the ellipsoid to shrink the confidence interval. This will reduce the exploration which Thompson sampling provides and can possibly lead to a faster convergence although this is not theoretically guaranteed.
\end{enumerate}

\subsection{Experimental Details}
For both reward setups, we used $N = 45$ base arms. The cardinality constraints were chosen as $K = 2,4,8$ for the easy construction of Hadamard matrices for the CSAR. We ran experiments on a time horizon of $10^6$. We ran each algorithm 25 times except LinTS (2 times) due to its excessively high runtime. We calculated the average regret and the maximum and minimum values of the cumulative regret of each algorithm.


\subsection{Results and Discussion}


\begin{figure*}[!t]
    \centering
    \begin{subfigure}[b]{0.32\textwidth}
        \centering
        \includegraphics[width=\textwidth]{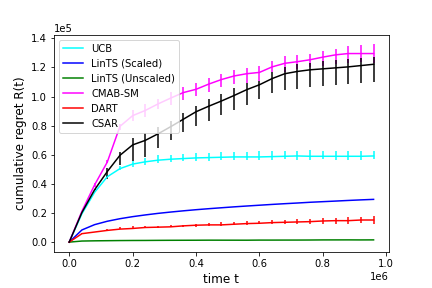}
        \caption{$K=2$.}
        \label{fig:K_2_mean_linTS}
    \end{subfigure}
    \begin{subfigure}[b]{0.32\textwidth}
        \centering
        \includegraphics[width=\textwidth]{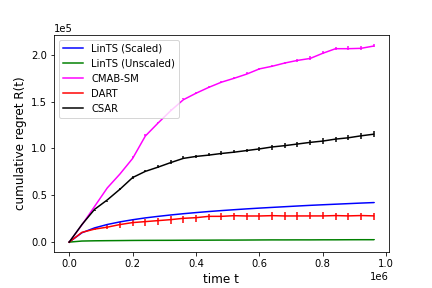}
        \caption{$K=4$.}
        \label{fig:K_4_mean_linTS}
    \end{subfigure}
    \begin{subfigure}[b]{0.32\textwidth}
        \centering
        \includegraphics[width=\textwidth]{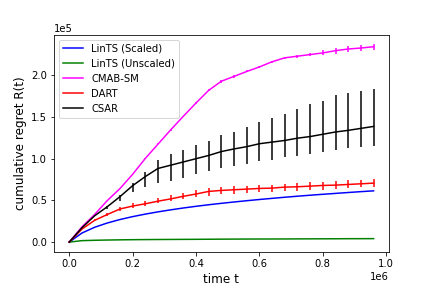}
        \caption{$K=8$.}
        \label{fig:K_8_mean_linTS}
    \end{subfigure}
    \caption{Regret plots for joint rewards as the mean of individual arm rewards.}
    \label{fig:mean_reward_plots}
\end{figure*}

\begin{figure*}[!t]
    \centering
    \begin{subfigure}[b]{0.32\textwidth}
        \centering
        \includegraphics[width=\textwidth]{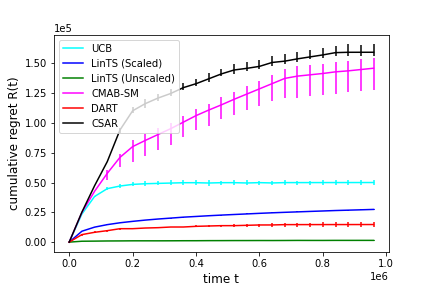}
        \caption{$K=2$.}
        \label{fig:K_2_quad}
    \end{subfigure}
    \begin{subfigure}[b]{0.32\textwidth}
        \centering
        \includegraphics[width=\textwidth]{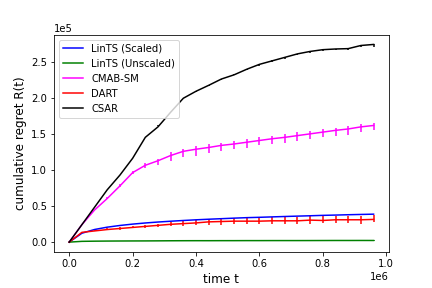}
        \caption{$K=4$.}
        \label{fig:K_4_quad}
    \end{subfigure}
    \begin{subfigure}[b]{0.32\textwidth}
        \centering
        \includegraphics[width=\textwidth]{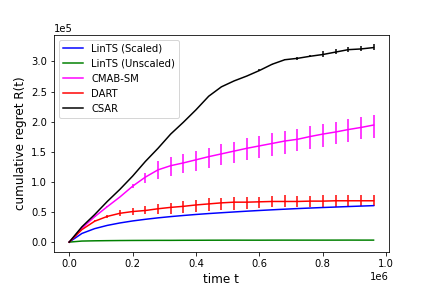}
        \caption{$K=8$.}
        \label{fig:K_8_quad}
    \end{subfigure}
    \caption{Regret plots for joint rewards as a quadratic function of individual arm rewards.} \label{fig:quad_reward_plots}
\end{figure*}

\Cref{fig:mean_reward_plots,fig:quad_reward_plots} depict the average (across runs) cumulative regret curves for \NAM\ (in red) and baseline algorithms for the linear and quadratic reward setups respectively. The maximum and minimum values of the cumulative regret of each algorithm are also represented using error bars in the plots.

From Figure \ref{fig:mean_reward_plots}, we note that, in terms of cumulative regret, for all $K=2,4,8$, \NAM\ performs significantly better than the baseline algorithms except for LinTS when the joint reward is the mean of the individual arm rewards. DART performs better than CSAR because after updating $\Delta$, CSAR generates fresh $\frac{K^2}{\Delta^2}$ samples instead of using previous samples to improve estimates. Hence, although CSAR is order optimal, we observed that the performance deteriorates in practice. DART performs better than UCB as the latter has a much larger action space compared to DART due to combinatorial explosion. Note that, due to the combinatorial explosion of the actions possible, we ran UCB only for $K=2$. 

From Figure \ref{fig:mean_reward_plots}, we next comment on the empirical performance of DART compared to the two versions of LinTS we considered. 
When compared to LinTS (Scaled), DART performs better for $K=2,4$. However, as $K$ increases LinTS (Scaled) performs better than \NAM. This is expected as the regret scaling of $O(\sqrt{N})$ for LinTS is better than the regret scaling of $O(K\sqrt{K})$ for DART. For the linear joint reward setup under consideration, we note that LinTS is better than CSAR as the latter discards samples after every round making it empirically inefficient. Furthermore, we observe that LinTS (Unscaled) performs better than DART. We attribute this better performance to the problem parameters. We note that the expected rewards for individual arms for the good arms are close to 1. This results in a low variance of the rewards and significant exploration with a scaling factor of N is not necessary. However, DART uses the confidence intervals constructed using the highest possible variance for the arm rewards, and it explores more even when it is not required. This worst-case variance-based algorithm vs. identifying the variance of problem setup using Bernstein inequality is a possible future extension for this work.

From Figure \ref{fig:quad_reward_plots}, we note that, in terms of cumulative regret, for all $K=2,4,8$, \NAM\ performs significantly better than the baseline algorithms except for LinTS when the joint reward is a quadratic function of the individual arm rewards. From Figure \ref{fig:quad_reward_plots}, we next comment on the empirical performance of DART compared to the two versions of LinTS we considered. When compared to LinTS (Scaled), DART performs better for $K=2,4$. However, as $K$ increases LinTS (Scaled) performs better than \NAM. Furthermore, we observe that LinTS (Unscaled) performs better than DART. The explanation for these observations is similar to the mean joint rewards setup.

For empirical results when the joint reward is the maximum of individual arm rewards, refer to Appendix \ref{sec:max_experiments}.



}
\color{black}{}
\section{Conclusion}
We considered the problem of combinatorial multi-armed bandits with non-linear rewards, where the agent chooses $K$ out of $N$ arms in each time-step and receives an aggregate reward. We obtained a lower bound of $\Omega(\sqrt{NKT})$ for the linear case with possibly correlated rewards. We proposed a novel algorithm, called \NAM, which is computationally efficient and has a space complexity which is linear in number of base arms. We analyzed the algorithm in terms of regret bound, and show that it is upper bounded by $\Tilde{\mathcal{O}}(K\sqrt{NKT})$. This regret bound  is only loose by a factor $O(K\sqrt{\log(NT)})$ compared to the lower bound of $\Omega(\sqrt{NKT})$ for bandit setup with correlated rewards . \NAM\ works efficiently for large $N$ and $K$ and outperforms existing methods empirically. 

\appendix
\newpage
{\color{black}
\section{Lower bound on  the top-$K$ subset selection problem} \label{app:lower_bound}
We start by noting that the reward function $f(\bm{X}) = 1/K\sum_{i=1}^K \bm{X}(i)$ satisfies  Assumptions $1-3$. Assumption $1$ holds from the fact that addition is commutative. Assumption $2$ holds from the fact that Expectation is linear. Lastly, the Corollary from the Assumption $3$ follows from the linearity of the sum. The next part of the proof follows the lines of the proof of Lemma 4 in \citep{cohen2017tight}, with relevant modifications based on the change in problem setup.

\begin{proof}[Proof of Theorem \ref{thm:lower_bound}]
Consider a setup with rewards $X_{i,t}' = 1/2 + \epsilon\bm{1}_{\{i\in{\bm{a}^*}\}} + Z_t$, where $\bm{a}^*$ is the best subset and $Z_t$ follows a Gaussian distribution with mean $0$ and variance $\sigma^2$. Let $\bm{a}^* = (i_1^*, i_2^*, \cdots, i_K^*)\in \mathcal{N}$  such that $i_j^* < i_k^* \forall 1\leq j< k\leq K$ be chosen uniformly randomly from $\mathcal{N}$. Let $T_1, \cdots, T_k$ be random variables such that $T_j$ is the number of times the agent plays $\bm{a}_t$ such that $i_j^* \in \bm{a}_t$. For each $\bm{a}\in \mathcal{N}$, let $\mathbb{P}_{\bm{a}}$ and $\mathbb{E}_{\bm{a}}$ be the probability distribution and the expectation with respect to the marginal distributions under which $\bm{a}^* = \bm{a}$. Then,
\begin{align}
    R &= \bm{E}\left[\max_{\bm{a}}\sum_{t=1}^T\frac{1}{K}\left(\sum_{i=1}^K (X_{\bm{a}(i)}' - X_{\bm{a}_t(i)}')\right)\right]\\
        &\geq \frac{1}{K}\bm{E}\left[\sum_{t=1}^T\left(\sum_{i=1}^K (X_{\bm{a}^*(i)}' - X_{\bm{a}_t(i)}')\right)\right]\\
      &= \frac{1}{K}\frac{1}{\binom{N}{K}}\sum_{\bm{a}\in\mathcal{N}}\mathbb{E}_{\bm{a}}\left[\sum_{t=1}^T\left(\sum_{i=1}^K (X_{\bm{a}(i)}' - X_{\bm{a}_t(i)}')\right)\right]\\
      &= \frac{1}{K}\frac{1}{\binom{N}{K}}\sum_{\bm{a}\in\mathcal{N}}\epsilon\mathbb{E}_{\bm{a}}\left[\sum_{j=1}^K(T-T_j)\right]\\
      &= \frac{1}{K}\epsilon\left(KT -\sum_{j=1}^K\frac{1}{\binom{N}{K}}\sum_{\bm{a}\in\mathcal{N}}\mathbb{E}_{\bm{a}}[T_j]\right)\\
      &= \epsilon\left(T -\frac{1}{\binom{N}{K}}\sum_{\bm{a}\in\mathcal{N}}\mathbb{E}_{\bm{a}}[T_j]\right)\label{eq:lower_bound_part_1}
\end{align}
where the Equation \eqref{eq:lower_bound_part_1} comes from the fact that the distributions of the optimal arms are identical and hence $\mathbb{E}_{\bm{a}}[T_j]$ is the same for all $j$. We now need to upper bound $\mathbb{E}_{\bm{a}}[T_j]$ for $j$.

For every $\bm{a}\in \mathcal{N}$ and $j\in[K]$, we introduce a new distribution, which is same as $\mathbb{P}_{\bm{a}}$ except that the reward of $i_j^*$ is also $1/2 +Z_t$. We refer to these laws by $\mathbb{P}_{\bm{a},-j}$ and $\mathbb{E}_{\bm{a},-j}$. Let $\lambda_t$ be the reward obtained at time $t$, and $\lambda^{(t)} = (\lambda_1, \cdots, \lambda_t)$ be the sequence of rewards obtained up to and including time $t$. Then, since $\lambda^{(T)}$ determines the actions of the learner over the entire game, and using Pinsker's inequality,
\begin{align}
    \mathbb{E}_{\bm{a}}[T_j] - \mathbb{E}_{\bm{a}, -j}[T_j] &\leq T\cdot D_{TV}\left(\mathbb{P}_{\bm{a}, -j}[\lambda^{(T)}] ,  \mathbb{P}_{\bm{a}}[\lambda^{(T)}]\right)\\
    &\leq T\sqrt{\frac{1}{2}D_{KL}\left(\mathbb{P}_{\bm{a}, -j}[\lambda^{(T)}] || \mathbb{E}_{\bm{a}}[\lambda^{(T)}]\right)}\label{eq:pinskers_bound}
\end{align}
{\color{black}
where $D_{TV}(p,q)$ is the total variation distance between distribution $p$ and distribution $q$ and $D_{KL}(p||q)$ is the Kullback-Liebler (KL) divergence between distribution $p$ and distribution $q$.
}

Now, from the chain rule of KL-divergence, $D_{KL}\left(\mathbb{P}_{\bm{a}, -j}[\lambda^{(T)}] || \mathbb{E}_{\bm{a}}[\lambda^{(T)}]\right)$ becomes
\begin{align}
    \sum_{t=1}^T\mathbb{E}_{\lambda^{(t-1)}\sim \mathbb{P}_{\bm{a}, -j}}\left[D_{KL}\left({\mathbb{P}_{\bm{a},-j}[\lambda_t|\lambda^{(t-1)}]}||{\mathbb{P}_{\bm{a}}[\lambda_t|\lambda^{(t-1)}]}\right)\right]\label{eq:chain_rule_kl_div}
\end{align}

Consider a single term in the sum in Equation \eqref{eq:chain_rule_kl_div}. If $i_j^*\notin\bm{a}_t$, then the reward obtained under $\bm{P}_{\bm{a}}$ and $\bm{P}_{\bm{a},-j}$ are the same, and the KL divergence is $0$. If $i_j^*\in\bm{a}_t$, then the rewards under $\bm{P}_{\bm{a}}$ and $\bm{P}_{\bm{a},-j}$ are both Gaussian with $\epsilon/K$ far means and variance of $\sigma^2$. Hence, we have
\begin{align}
    D_{KL}\left({\mathbb{P}_{\bm{a},-j}[\lambda_t|\lambda^{(t-1)}]}||{\mathbb{P}_{\bm{a}}[\lambda_t|\lambda^{(t-1)}]}\right) \leq \frac{1}{2}\left(\frac{\epsilon}{K}\right)^2\frac{1}{\sigma^2}.
\end{align}
Using the obtained KL-divergence in Equation \eqref{eq:chain_rule_kl_div} and subsequently in Equation \eqref{eq:pinskers_bound} we get,
\begin{align}
    D_{KL}\left(\mathbb{P}_{\bm{a}, -j}[\lambda^{(T)}] || \mathbb{E}_{\bm{a}}[\lambda^{(T)}]\right) &\leq \sum_{t=1}^T\mathbb{P}_{\bm{a},-j}[i_j^*\in\bm{a}_t]\frac{\epsilon^2}{2K^2\sigma^2} = \frac{\epsilon^2}{2K^2\sigma^2}\mathbb{E}_{\bm{a},-j}[T_j],\text{ and }\\
    \mathbb{E}_{\bm{a}}[T_j] &\leq \mathbb{E}_{\bm{a},-j}[T_j] + \frac{T\epsilon}{2K\sigma}\sqrt{\mathbb{E}_{\bm{a},-j}[T_j]}
\end{align}

We now want to upper bound $\mathbb{E}_{\bm{a}\sim\mathbb{U}(\mathcal{N})}[\mathbb{E}_{\bm{a},-j}[T_j]]$ to proceed from Equation \eqref{eq:lower_bound_part_1}. Note that at each time step, we play $K$ arms and hence collectively over $N$ arms, we play all arms exactly $KT$ times. Further, we fixed the distribution for all $j\in[N]$, we have: 
\begin{align}
    \mathbb{E}_{\bm{a}^*\sim\mathbb{U}(\mathcal{N})}[\mathbb{E}_{\bm{a}^*,-j}[T_j]] &= \frac{1}{N}\sum_{j=1}^N\mathbb{E}_{\bm{a}^*\sim\mathbb{U}(\mathcal{N})}[\mathbb{E}_{\bm{a}^*,-j}[T_j]]\\
    &= \frac{1}{N}\mathbb{E}_{\bm{a}^*\sim\mathbb{U}(\mathcal{N})}[\mathbb{E}_{\bm{a}^*,-j}[\sum_{j=1}^NT_j]]\\
    &= \frac{1}{N}\mathbb{E}_{\bm{a}^*\sim\mathbb{U}(\mathcal{N})}[\mathbb{E}_{\bm{a}^*,-j}[KT]]\\
    &= \frac{1}{N}KT\\
    &= \frac{KT}{N}
\end{align}

Now we have $K \le N/2$ which implies $TK/(N)\le T/2$ which gives:
\begin{align}
    \frac{1}{\binom{N}{K}}\sum_{\bm{a}\in\mathcal{N}}\mathbb{E}_{\bm{a}\in\mathcal{N}}[T_j]&\leq \frac{1}{\binom{N}{K}}\sum_{\bm{a}\in\mathcal{N}}\mathbb{E}_{\bm{a},-j}[T_j] + \frac{1}{\binom{N}{K}}\sum_{\bm{a}\in\mathcal{N}}\frac{\epsilon}{2K\sigma}\sqrt{\mathbb{E}_{\bm{a},-j}[T_j]}\label{eq:apply_Cauchy_Schwarz}\\
    &\leq\frac{1}{\binom{N}{K}}\sum_{\bm{a}\in\mathcal{N}}\mathbb{E}_{\bm{a},-j}[T_j] + \frac{\epsilon}{2K\sigma}\sqrt{\frac{1}{\binom{N}{K}}\sum_{\bm{a}\in\mathcal{N}}\mathbb{E}_{\bm{a},-j}[T_j]}\\
    &\leq \frac{TK}{N} + \frac{\epsilon}{2K\sigma}\sqrt{\frac{TK}{N}}\\
    &\leq \frac{T}{2} + \frac{\epsilon}{2K\sigma}\sqrt{\frac{TK}{N}}\\
    &\leq \frac{T}{2} + \frac{\epsilon}{2\sigma}\sqrt{\frac{T}{NK}}\\
\end{align}
where Equation \eqref{eq:apply_Cauchy_Schwarz} is obtained from the Cauchy Schwarz inequality.

Substituting this in Equation \eqref{eq:lower_bound_part_1}, we get
\begin{align}
    R\geq \epsilon T\left(\frac{1}{2}-\frac{\epsilon}{2\sigma}\sqrt{\frac{T}{NK}}\right)
\end{align}
for all values of $\epsilon$.  Choosing $\epsilon = \frac{\sigma}{2}\sqrt{\frac{NK}{T}}$, we have
\begin{align}
    R \geq \frac{\sigma }{8}\sqrt{NKT}
\end{align}

\end{proof}


}
\newpage
\section{Proof of Lemma \ref{lem:expected_Z_lemma}}\label{sec:expected_Z_lemma_proof}

\begin{proof}
We first show the upper bound. The cardinality of both $\mathcal{N}_e(i)$ and $\mathcal{N}_e(j)$ is $\binom{|\mathcal{N}_e|-1}{K_e-1}$ as we have fixed one of the $K_e$ places for arm $i$ and now we can fill only $K_e-1$ places from the available $|\mathcal{N}_e|-1$ arms.  { Algorithm \ref{alg:nsras} partitions a random, uniformly distributed  permutation over $\mathcal{N}_e$, so all actions $\bm{a}\in\mathcal{N}_e(i)$ are equally likely, and likewise for $\bm{a}\in\mathcal{N}_e(j)$. } Taking the expectation over the actions played and the reward obtained, we get the expected value of $Z_{i,j}(e)$ as
\begin{align}
    \mathbb{E}\left[Z_{i,j}(e)|i,j\in\mathcal{N}_e\right] &= \mathbb{E}\left[r_{\bm{a}_i}(e) - r_{\bm{a}_j}(e)|i,j\in\mathcal{N}_e\right]\\ &=\frac{1}{\binom{|\mathcal{N}_e|-1}{K_e-1}}\left(\sum_{\bm{a}\in\mathcal{N}_e(i)}\mu_{\bm{a}} - 
    \sum_{\bm{a}\in\mathcal{N}_e(j)}\mu_{\bm{a}}\right)\label{eq:exp_over_a_and_e}\\
    &\leq \frac{1}{\binom{|\mathcal{N}_e|-1}{K_e-1}}\binom{|\mathcal{N}_e|-2}{K_e-1}U\Delta_{i,j}\label{eq:Z_upper_bnd}\\
    %
    %
    &=\frac{|\mathcal{N}_e|-K_e}{|\mathcal{N}_e|-1}U\Delta_{i,j} \leq U\Delta_{i,j} \label{eq:Z_gap_upper}.
\end{align} 
Equation \eqref{eq:exp_over_a_and_e} is obtained by the linearity of expectation and taking the expectation over rewards of uniformly distributed actions $\bm{a}_i$ and $\bm{a}_j$. Equation \eqref{eq:Z_upper_bnd} is obtained by noting that there exist exactly $\binom{|\mathcal{N}_e|-2}{K_e-1}$ actions where arm $i$ is replaced by arm $j$. From Assumption \ref{continuous_assumption} of Lipschitz continuity, the difference between the expected reward of those actions is bounded by $U\Delta_{i,j}$. The remaining actions contain both arms $i$ and $j$, thus are in both $\mathcal{N}_e(i)$ and $\mathcal{N}_e(j)$, and so cancel out. Equation \eqref{eq:Z_gap_upper} comes from simplifying the fraction with binomial and noticing that $K_e\geq 1$. This proves the upper bound.

Similarly, we obtain the lower bound using Assumption \ref{inverse_continuity}
\begin{align}
    \mathbb{E}\left[Z_{i,j}(e)|i,j\in\mathcal{N}_e\right] &= \mathbb{E}\left[r_{\bm{a}_i}(e) - r_{\bm{a}_j}(e)|i,j\in\mathcal{N}_e\right]\\ &=\frac{1}{\binom{|\mathcal{N}_e|-1}{K_e-1}}\left(\sum_{\bm{a}\in\mathcal{N}_e(i)}\mu_{\bm{a}} - 
    \sum_{\bm{a}\in\mathcal{N}_e(j)}\mu_{\bm{a}}\right)\\
    &\geq \frac{\Delta_{i,j}}{U}\frac{1}{\binom{|\mathcal{N}_e|-1}{K_e-1}}\binom{|\mathcal{N}_e|-2}{K_e-1}\label{eq:Z_lower_bnd}\\
    %
    %
    &= \frac{\Delta_{i,j}}{U}\frac{|\mathcal{N}_e|-K_e}{|\mathcal{N}_e|-1} \geq \frac{\Delta_{i,j}}{UK_e } \geq \frac{\Delta_{i,j}}{UK} \label{eq:Z_gap_lower}.
\end{align} 
{\color{black}Equation \eqref{eq:Z_gap_lower} is obtained from Corollary \ref{inverse_continuity}.} The difference between the expected reward of the actions is lower bounded by $\frac{\Delta_{i,j}}{U}$. Equation \eqref{eq:Z_gap_lower} comes from noting that $K_e(|\mathcal{N}_e| - K_e)\geq |\mathcal{N}_e| - 1$. This proves the lower bound.
\end{proof}
\section{Proof of Lemma \ref{lem:num_epochs_lemma}}\label{sec:proof_num_epochs_lemma}

Before proving the result, we first state the Azuma-Hoeffding Lemma which we use to calculate the concentration inequalities.  

\begin{lemma}[Azuma-Hoeffding {\citep[Chapter 3]{bercu2015concentration}}]\label{lem:azuma_hoeffding}
If $\{W_n\}$ is a zero-mean martingale process with almost surely bounded increments $|W_n-W_{n-1}| \leq C$, then for any $\delta>0$ with probability at least $1-\delta$, $|W_n| \leq C\sqrt{2n \log(2/\delta)}$.
\end{lemma}

{\color{black}
\begin{proof}[Proof of Lemma \ref{lem:num_epochs_lemma}]
Notice that the difference of estimates of arms $i$ and $j$, $\hat{\mu}_i-\hat{\mu}_{j}$, at epoch $e$ is $(\sum_{e'=1}^e Z_{i,j}(e'))/e$. Also, the number of times arm $i$ and arm $j$ are sampled is the same as the epoch counter $e$ as each arm is sampled only once in an epoch. Further, the \NAM\ algorithm collects $e = \frac{32\log(2NT)}{\Delta^2}$ samples of an arm before decaying $\Delta$ to $\Delta/2$. Hence, for an $\Delta$ such that $\Delta < \hat{\mu}_i - \hat{\mu}_j\le 2\Delta$, the \NAM\ algorithm must have collected at least $e = \frac{32\log (2NT)}{(2\Delta)^2}$ samples for arms $i$ and $j$. Further, the number of samples cannot exceed $\frac{32\log (2NT)}{\Delta^2}$.

Using the concentration bound in Lemma \ref{lem:azuma_hoeffding} on $|Y_{i,j}(e)|$ with $C=2$, then with probability at least $1-NT$, we get
\begin{align}
    \Big|\frac{Y_{i,j}(e)}{e}\Big| %
    &\leq  \frac{2}{e}\sqrt{2e\log{(2NT)}} \nonumber\\
    &=  \sqrt{\frac{8\log{(2NT)}}{e}} \nonumber\\
    &\leq \sqrt{\frac{8\log{(2NT)}(2\Delta)^2}{32\log{(2NT)}}}\label{eq:min_epochs} \\
    &\leq \Delta, \nonumber
\end{align}



{
Plugging in the value of $Y_{i,j}(e)$ from Equation \eqref{eq:martingale_def}, we have 
\begin{equation}
    \frac{1}{e}\Big|\sum_{e'=1}^eZ_{i,j}(e') - \sum_{e'=1}^e\mathbb{E}[Z_{i,j}(e')]\Big|\leq \Delta\label{eq:Z_conc}
\end{equation}
}



From the statement of Lemma \ref{lem:num_epochs_lemma}, we have $\hat{\mu}_i-\hat{\mu}_j < 2\Delta$. Combining with the definition of $Y_{i,j}(e)$ in Equation \eqref{eq:martingale_def}, we have
\begin{equation}
    \hat{\mu}_i-\hat{\mu}_{j} < 2\Delta
\end{equation}

Subtracting $\frac{1}{e}\sum_{e'=1}^e\mathbb{E}[Z_{i,j}(e')]$ from both sides, we have
\begin{equation}
    \hat{\mu}_i-\hat{\mu}_{j} - \frac{1}{e}\sum_{e'=1}^e\mathbb{E}[Z_{i,j}(e')] < 2\Delta-\frac{1}{e}\sum_{e'=1}^e\mathbb{E}[Z_{i,j}(e')]
\end{equation}

Again using $\hat{\mu}_i-\hat{\mu}_j = (\sum_{e'=1}^e Z_{i,j}(e'))/e$,  
\begin{equation}
    \frac{1}{e}\sum_{e'=1}^e Z_{i,j}(e') - \frac{1}{e}\sum_{e'=1}^e\mathbb{E}[Z_{i,j}(e')] < 2\Delta-\frac{1}{e}\sum_{e'=1}^e\mathbb{E}[Z_{i,j}(e')]
\end{equation}

By definition of $Y_{i,j}(e)$, we obtain
\begin{align}
    \frac{Y_{i,j}(e)}{e} &< 2\Delta-\frac{1}{e}\sum_{e'=1}^e\mathbb{E}[Z_{i,j}(e')] \\
    &< 2\Delta-\frac{\Delta_{i,j}}{UK}\label{eq:2Delta_gap}
\end{align}

Combining Equation \eqref{eq:2Delta_gap} and the negative part of left hand side of Equation \eqref{eq:Z_conc}, we have
\begin{align}
    -\Delta \leq \frac{Y_{i,j}(e)}{e} &< 2\Delta-\frac{\Delta_{i,j}}{UK}\\
    \Rightarrow  \frac{\Delta_{i,j}}{UK} &< 3\Delta\\
    \Rightarrow  \Delta &> \frac{\Delta_{i,j}}{3UK}
\end{align}

Combining Lemma \ref{lem:azuma_hoeffding} and the minimum value of $\Delta$ we obtain the required upper bound on the number of epochs to separate arm $i$ and arm $j$.
\end{proof}
}
\section{Proof of Lemma \ref{lem:action_gap_permutation}}\label{sec:proof_action_gap_permute}
\begin{proof}
Let $\bm{X}_{\bm{a}} = (X_{a_1}, \cdots, X_{a_K})$ be the random vector of the rewards of the arms in action $\bm{a}$. Also, let $\bm{X}_{\bm{a}^*} = (X_{1}, \cdots, X_{K})$ be the random vector of the rewards of the optimal arms. Then from Assumption \ref{continuous_assumption}, we have

\begin{align}
        \big|\ \mu_{{\bf a}} -\ \mu_{{\bf a}^*} \big| \ 
        &\leq \ U \min_{\pi \in \Pi}{\big|\big|}\mathbb{E}[{\bm X}_{{\bf a}}] - \pi(\mathbb{E}[{\bm X}_{{\bf a}^*}]){\big|\big|_1}\label{eq:appd:lem4:L2L1}\\
        &= U \min_{\pi \in \Pi} \sum_{i=1}^K\bigg|\mathbb{E}[X_{a_i}] - \mathbb{E}[X_{\pi(i)}]\bigg| \label{eq:appd:lem4:evalL1}\\
        &\leq \ U \sum_{i=1}^K\big|\mathbb{E}[X_{a_i}] - \mathbb{E}[X_{\pi'(i)}]\big|,\label{eq:appd:lem4:bnd}
\end{align} %
where \eqref{eq:appd:lem4:L2L1} uses the property that the $\ell_2$ norm is upper bounded by the $\ell_1$ norm, \eqref{eq:appd:lem4:evalL1} evaluates the $\ell_1$ norm and uses the property that $\bm{a}^*$ contains arms $1$ through $K$, and \eqref{eq:appd:lem4:bnd} holds for any permutation $\pi'$ of $\{1, \cdots, K\}$ which matches arms in $\bm{a}^*$ with corresponding $\bm{a}$, so  $\pi'(i) = a_i$ for any arm  $a_i\leq K$.

\end{proof}

\newpage
\section{Proof of Lemma \ref{lem:arm_regret_bound}} \label{sec:proof_arm_regret_bound}
\begin{proof}
We begin by counting the cumulative pseudo-regret incurred from actions involving sub-optimal arms.  Let $i\geq K+1$ denote a sub-optimal arm. Similar to \eqref{eq:regret}, let $R_i$ denote the cumulative pseudo-regret incurred from all of the actions $\bm{a}_t$ with sub-optimal arm $i\in \bm{a}_t,$
\begin{equation}
    R_i  = \mathbb{E}_{\bm{a}_1, \cdots, \bm{a}_T, }\left[\sum_{t=1}^T \left(\mu_{\bm{a}^*} - \mu_{\bm{a}_t }\right)\mathbbm{1}_{i\in \bm{a}_t}\right] \label{eq:Ri}.
\end{equation}

Similarly, let $R_{i,j}$ denote the cumulative pseudo-regret incurred from all of the actions $\bm{a}_t$ with sub-optimal arm $i\in \bm{a}_t,$ with respect to the  action where an optimal arm $j\not \in \bm{a}_t$ is swapped with arm $i$,  
\begin{equation}
    R_{i,j}  = \mathbb{E}_{\bm{a}_1, \cdots, \bm{a}_T, }\left[\sum_{t=1}^T  \left(  \mu_{\{j\} \cup\bm{a}_t \setminus \{i\}  } - \mu_{\bm{a}_t }\right) \mathbbm{1}_{i\in \bm{a}_t}\mathbbm{1}_{j \not \in \bm{a}_t}\right] \label{eq:Rij}.
\end{equation}
%
%


Similar to \cite{rejwan2020top}, we group the top-$K$ arms into two groups  %
based on how easy it is to separate them from the $(K+1)$th arm, $R^< = \{j\mid 1\leq j \leq K, \ \Delta_{j, K+1}\leq\Delta_{K,i}\}$ and $R^> = \{j\mid 1\leq j \leq K, \ \Delta_{j, K+1}>\Delta_{K,i}\}$. This gives the bound on the regret from arm $i, R_i$ as
\begin{align}
    R_i &\leq \sum_{j\in R^<} R_{i,j}+\sum_{j\in R^>} R_{i,j}
\end{align}

We now calculate the regret for both groups $R^<$ and $R^>$ separately in the following cases:
\begin{enumerate}
    \item \textbf{Regret by replacing arm} $j\in R^<$: Note that $\Delta_{j,i}<\Delta_{j,K+1}+\Delta_{K,i}$ as the gap for $\Delta_{K,K+1}$ is counted twice. Hence we get the upper bound on regret as, %
    \begin{align}
        \sum_{j\in R^<} R_{i,j} &\leq \frac{288U^2K^2\log{(2/\delta)}}{\Delta_{K,i}^2}\max_{j\in R^<}U\Delta_{j,i}\\
                &\leq \frac{288U^2K^2\log{(2/\delta)}}{\Delta_{K,i}^2}\max_{j\in R^<}U(\Delta_{j,K+1}+\Delta_{K,i})\\
                &\leq \frac{288U^2K^2\log{(2/\delta)}}{\Delta_{K,i}^2}U(2\Delta_{K,i}) = \frac{576U^3K^2\log{(2/\delta)}}{\Delta_{K,i}}
    \end{align}
    \item \textbf{Regret by replacing arm} $j\in R^>$: Note that since $\Delta_{j,K+1} > \Delta_{K, i}$, arm $j$ will move to the accept set before arm $i$ is rejected with probability at least $1-2\delta$. Once the algorithm moves arm $j$ to the accept set $\mathcal{A}_e$, it will not suffer any regret from replacing arm $j$. Let, $l = \arg\min_{j\in R^>\Delta_j}$, then we can bound the regret from arms in $R^>$ using following inequalities.
    \begin{align}
        \sum_{j\in R^>} R_{i,j} &\leq \frac{288U^2K^2\log{(2/\delta)}U\Delta_{1,i}}{\Delta_{1,K+1}^2}
        + \sum_{j=2}^l 288U^2K^2\log{(2/\delta)}\left(\frac{U\Delta_{j,i}}{\Delta_{j-1,K+1}^2}-\frac{U\Delta_{j,K+1}}{\Delta_{j-1,K+1}^2}\right)\\
        &\leq 288U^3K^2\log{(2/\delta)}\Bigg( \left(\sum_{j=1}^{l-1} \frac{\Delta_{j,i}-\Delta_{j+1,i}}{\Delta_{j,K+1}^2}\right) + \left(\frac{\Delta_{l,i}}{\Delta_{l,K+1}^2}\right)\Bigg)\\
        &\leq 288U^3K^2\log{(2/\delta)}\Bigg( \left(\sum_{j=1}^{l-1} \frac{\Delta_{j,K+1}-\Delta_{j+1,K+1}}{\Delta_{j,K+1}^2}\right) + \left(\frac{\Delta_{l,K+1}+\Delta_{K,i}}{\Delta_{l,K+1}^2}\right)\Bigg)\\
        &\leq 288U^3K^2\log{(2/\delta)}\Bigg( \left(\int_{\Delta_{l,K+1}}^{\Delta_{1,K+1}}\frac{1}{x^2}dx\right) + \left(\frac{2\Delta_{l,K+1}}{\Delta_{l,K+1}^2}\right)\Bigg)\\
        &= 288U^3K^2\log{(2/\delta)}\Bigg(\frac{1}{\Delta_{l,K+1}} -\frac{1}{\Delta_{1,K+1}}+\frac{2}{\Delta_{l,K+1}}\Bigg)\\
        &\leq 288U^3K^2\log{(2/\delta)}\left(\frac{3}{\Delta_{l,K+1}}\right) \leq \frac{864K^2U^3\log{(2/\delta)}}{\Delta_{K,i}}
    \end{align}
\end{enumerate}
Summing up the regrets for $R^<$ and $R^>$, we get total regret for sampling arm $i$ to be bounded as
\begin{align}
    R_i = \frac{1440U^3K^2\log{(2/\delta)}}{\Delta_{K,i}}
\end{align}
\end{proof}
\newpage
\section{Proof of Theorem \ref{thm:main_theorem}}\label{sec:proof_main_theorem}
\begin{proof}
We note that there are three sources of regret for the \NAM. 
\begin{enumerate}
    \item The first is that regret will accumulate while eliminating sub-optimal arms. From Lemma \ref{lem:arm_regret_bound}, the regret accumulated while eliminating sub-optimal arms is bounded as
    \begin{align}
        \sum_{i=K+1}^N\frac{1440K^2U^3\log{2/\delta}}{\Delta_{K,i}}
    \end{align}
    \item The second is when the algorithm is not able to move optimal arms from $\mathcal{N}_e$ to the accept set $\mathcal{A}_e$ or move sub-optimal arms to the reject set $\mathcal{R}_e$ because of separability. That is, $\Delta_{i, K+1}< \lambda$ for some optimal arm $i: 1\leq i\leq K$, or $\Delta_{K, i}$ for some sub-optimal arm $i: K+1\leq i\leq N$. 

    To bound the regret, we will apply Lemma \ref{lem:action_gap_permutation} with a sub-optimal action from a ``worst-case'' scenario where the top $K$ arms are not separable from the $(K+1)$st, so $\Delta_{i,K+1}<\lambda$ for $i:\ 1\leq i \leq K$, and the $(K+1)$st through $(2K+1)$st arms are not separable from the $K$th, so $\Delta_{K,i}<\lambda$ for $i:\ K+2\leq i \leq 2K+1$.  In this scenario, the accept set remains empty.  Consider the action ${\bm a} = (K+2, \dots, 2K+1)$ formed by using only sub-optimal arms.  Note we can make the regret largest with $\Delta_{K,K+1}\approx 0$.  By construction, we have $\Delta_{1,2K+1} < 2\lambda$, e.g. those two arms cannot have means that are too far apart. Consequently, using Lemma \ref{lem:action_gap_permutation} for this action ${\bm a}$,  the expected (instantaneous) regret can be bounded as $UK(2\lambda)$.  Thus, we can bound the overall cumulative regret the algorithm will suffer from this issue as
    \begin{align}
        2TUK\lambda
    \end{align}
    \item The third source of regret is due to either if a sub-optimal arm $i: K+1\leq i\leq N$ is not moved to ``reject'' set despite $\Delta_{K, i} \geq \lambda$ or if an optimal arm $i: 1\leq i \leq K$ is not moved to ``accept'' set despite $\Delta_{i,K+1}>\lambda$, in the number of rounds calculated in Lemma \ref{lem:num_epochs_lemma} . Using the union bound and Lemma \ref{lem:azuma_hoeffding}, the probability of this event can be bounded using $N\times 1/(NT)$ for each arm moved to the corresponding incorrect set. 
    With the loose upper bound of $UK$ (since the difference in means of any pair of arms is at most $1$), the expected cumulative regret from this situation can be bounded as
    \begin{align}
        TUK\times N\frac{1}{NT} = UK
    \end{align}
\end{enumerate}
Thus, the total regret  of \NAM\ algorithm can be bounded as
\begin{align}
    R &\leq \sum_{i:\Delta_{K,i}\ge \lambda}\frac{1440U^3K^2\log{2NT}}{\Delta_{K,i}} + 2TUK\lambda + UK\\
    &\leq \frac{1440NU^3K^2\log{2NT}}{\lambda} + 2TUK\lambda + UK\label{eq:lambda_replace}.
\end{align}
Equation \ref{eq:lambda_replace} is obtained from the fact that the algorithm stops sampling arms if the gap is small and cannot be resolved.

Choosing $\lambda = U\sqrt{\frac{720NK\log{2NT}}{T}}$, we get the required regret bound.
\end{proof}
\section{Conversion to anytime algorithm}\label{sec:anytime_algorithm_proof}
Our proposed algorithm \NAM\ requires the time horizon $T$ as an input.  However, \NAM\ can be modified to not require knowledge of $T$. We use the standard doubling trick from Multi-Armed Bandit literature \cite{auer2010ucb,besson2018doubling}. To use the doubling trick, we start the algorithm from $T_0 = 0$. We then restart the algorithm after every $T_l = 2^l,\ l=1,2, \cdots$ time steps, till the algorithm reaches the unknown $T$. Each restart of the algorithm runs for $T_l - T_{l-1}$ steps with $T_0 = 0$ with $\lambda_l = \sqrt{\frac{1440 N\log {2N(T_l - T_{l-1})}}{K(T_l - T_{l-1})}}$

To show that the regret is bounded by $T^{1/2}$ for the doubling algorithm, we use Theorem 4 from \cite{besson2018doubling} which we state in the following lemma.

\begin{lemma}\label{doubling_trick_lemma} {\cite[Theorem 4]{besson2018doubling}}
If an algorithm $\mathcal{A}$ satisfies $R_T(\mathcal{A}_T) \leq cT^\gamma(\log T)^\delta + f(T)$, for $0< \gamma < 1$, $\delta\geq 0$ and for $c> 0$, and an increasing function $f(t) = o(t^\gamma(\log t)^\delta(\text{at } t\to\infty)$, then anytime version $\mathcal{A}' := \mathcal{DT}(\mathcal{A}, (T_i)_{i\in\mathbb{N}})$ with geometric sequence $(T_i)_{i\in\mathbb{N}}$ of parameters $T_0\in\mathbb{N}^*$, $b>1, (i.e.,T_i = \lfloor T_0b^i\rfloor)$ with the condition $T_0(b-1) > 1$ if $\delta > 0$ satisfies,
\begin{align}
    R_T(\mathcal{A}') \leq l(\gamma, \delta, T_0, b)cT^\gamma(\log T)^\delta + g(T),
\end{align}
with a increasing function $g(t) = o(t^\gamma(\log t)^\delta)$ and a constant loss $l(\gamma, \delta, T_0, b)>1$,
\begin{align}
    l(\gamma, \delta, T_0, b) := \left(\left(\frac{\log (T_0(b-1)+1)}{\log (T_0(b-1))}\right)^\delta\right)\times\frac{b^\gamma(b-1)^\gamma}{b^\gamma-1}
\end{align}
\end{lemma}

Using Lemma \ref{doubling_trick_lemma} for $b = 2, \gamma = 1/2, \delta = 1/2$, we can convert our algorithm to an anytime algorithm.
\newpage
\section{Additional Experiments} \label{sec:clear_experiments}

\subsection{Joint reward as $\max$ of arm rewards}\label{sec:max_experiments}
\begin{figure}[t]
    \centering
    \begin{subfigure}[b]{0.4\textwidth}
        \centering
        \includegraphics[width=\textwidth]{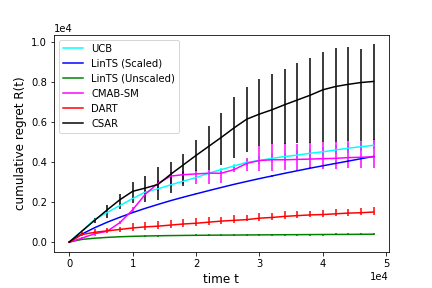}
        \caption{$K=2$}
        \label{fig:K_2_max}
    \end{subfigure}
    \begin{subfigure}[b]{0.4\textwidth}
        \centering
        \includegraphics[width=\textwidth]{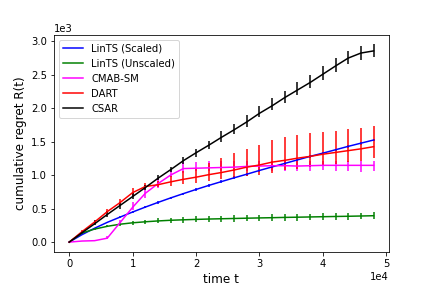}
        \caption{$K=4$}
        \label{fig:K_4_max}
    \end{subfigure}
    \caption{Regret plots for joint rewards as max of individual arm rewards}
    \label{fig:max_reward_plots_appndx}
\end{figure}

This experimental setup was the same as Section \ref{sec:experimental_results}, except the joint reward function is of the form $r(t) = \max_{i\in\bm{a}_t}{X_{i,t}}$.  The cumulative regrets of \NAM\ and other algorithms, averaged over 20 runs,  are shown in Figure \ref{fig:max_reward_plots_appndx}, we note that the performance of \NAM\ is significantly better than all other algorithms (except LinTS (Unscaled)) considered for $K=2$. The reason why We note that the CSAR algorithm performs the worst. We suspect this is because the CSAR algorithm is not able to approximate the $\max$ function with a linear model. For $K=4$, we note that the CMAB-SM algorithm (except LinTS (Unscaled)) performs the best. We suspect that is because CMAB-SM is able to eliminate arms faster for small $U$ with a divide-and-conquer approach. The reason why LinTS (Scaled) works the best is the same as explained in Section \ref{sec:experimental_results}.

\bibliography{refs}

@article{ auer02using,
  author = "Peter Auer",
  title = "Using Confidence Bounds for Exploitation-Exploration Trade-offs",
  journal = "Journal of Machine Learning Research",
  volume = "3",
  pages = "397-422",
  year = "2002"
}

@inproceedings{ dani08stochastic,
  author = "Varsha Dani and Thomas Hayes and Sham Kakade",
  title = "Stochastic Linear Optimization under Bandit Feedback",
  booktitle = "Proceedings of the 21st Annual Conference on Learning Theory",
  pages = "355-366",
  year = "2008"
}

@inproceedings{agarwal2021dart,
	title={Dart: Adaptive accept reject algorithm for non-linear combinatorial bandits},
	author={Agarwal, Mridul and Aggarwal, Vaneet and Umrawal, Abhishek Kumar and Quinn, Chris},
	booktitle={Proceedings of the AAAI Conference on Artificial Intelligence},
	volume={35},
	number={8},
	pages={6557--6565},
	year={2021}
}

@inproceedings{ abbasi-yadkori11improved,
  author = "Yasin Abbasi-Yadkori and David Pal and Csaba Szepesvari",
  title = "Improved Algorithms for Linear Stochastic Bandits",
  booktitle = "Advances in Neural Information Processing Systems 24",
  pages = "2312-2320",
  year = "2011"
}

@inproceedings{ filippi10parametric,
  author = "Sarah Filippi and Olivier Cappe and Aurelien Garivier and Csaba Szepesvari",
  title = "Parametric Bandits: The Generalized Linear Case",
  booktitle = "Advances in Neural Information Processing Systems 23",
  pages = "586-594",
  year = "2010"
}

@inproceedings{ jun17scalable,
  author = "Kwang-Sung Jun and Aniruddha Bhargava and Robert Nowak and Rebecca Willett",
  title = "Scalable Generalized Linear Bandits: Online Computation and Hashing",
  booktitle = "Advances in Neural Information Processing Systems 30",
  pages = "98-108",
  year = "2017"
}

@inproceedings{ li17provably,
  author = "Lihong Li and Yu Lu and Dengyong Zhou",
  title = "Provably Optimal Algorithms for Generalized Linear Contextual Bandits",
  booktitle = "Proceedings of the 34th International Conference on Machine Learning",
  pages = "2071-2080",
  year = "2017"
}

@inproceedings{rejwan2020top,
  title={Top-$ k $ Combinatorial Bandits with Full-Bandit Feedback},
  author={Rejwan, Idan and Mansour, Yishay},
  booktitle={Algorithmic Learning Theory},
  pages={752--776},
  year={2020}
}

@inproceedings{ gopalan14thompson,
  author = "Aditya Gopalan and Shie Mannor and Yishay Mansour",
  title = "Thompson Sampling for Complex Online Problems",
  booktitle = "Proceedings of the 31st International Conference on Machine Learning",
  pages = "100-108",
  year = "2014"
}

@inproceedings{agrawal2012analysis,
  title={Analysis of Thompson sampling for the multi-armed bandit problem},
  author={Agrawal, Shipra and Goyal, Navin},
  booktitle={Conference on Learning Theory},
  pages={39--1},
  year={2012}
}

@inproceedings{kveton2015tight,
  title={Tight regret bounds for stochastic combinatorial semi-bandits},
  author={Kveton, Branislav and Wen, Zheng and Ashkan, Azin and Szepesvari, Csaba},
  booktitle={Artificial Intelligence and Statistics},
  pages={535--543},
  year={2015}
}

@article{besson2018doubling,
  title={What Doubling Tricks Can and Can't Do for Multi-Armed Bandits},
  author={Besson, Lilian and Kaufmann, Emilie},
  year={2018}
}

@article{auer2010ucb,
  title={{UCB revisited: Improved regret bounds for the stochastic multi-armed bandit problem}},
  author={Auer, Peter and Ortner, Ronald},
  journal={Periodica Mathematica Hungarica},
  volume={61},
  number={1-2},
  pages={55--65},
  year={2010},
  publisher={Akad{\'e}miai Kiad{\'o}, co-published with Springer Science+ Business Media BV, Formerly Kluwer Academic Publishers BV}
}

@INPROCEEDINGS{1250942,
  author={ {Raymond Chi-Wing Wong} and  {Ada Wai-Chee Fu} and K. {Wang}},
  booktitle={Third IEEE International Conference on Data Mining}, 
  title={MPIS: maximal-profit item selection with cross-selling considerations}, 
  year={2003},
  volume={},
  number={},
  pages={371-378},
}

@inproceedings{zhang2012joint,
	title={Joint optimization of bid and budget allocation in sponsored search},
	author={Zhang, Weinan and Zhang, Ying and Gao, Bin and Yu, Yong and Yuan, Xiaojie and Liu, Tie-Yan},
	booktitle={Proceedings of the 18th ACM SIGKDD International Conference on Knowledge Discovery and Data Mining},
	pages={1177--1185},
	year={2012},
	organization={ACM}
}

@inproceedings{nuara2018combinatorial,
	title={A Combinatorial-Bandit Algorithm for the Online Joint Bid/Budget Optimization of Pay-per-Click Advertising Campaigns},
	author={Nuara, Alessandro and Trovo, Francesco and Gatti, Nicola and Restelli, Marcello and others},
	booktitle={Thirty-Second AAAI Conference on Artificial Intelligence},
	pages={1840--1846},
	year={2018}
}

@article{xiang2016joint,
	title={Joint latency and cost optimization for erasure-coded data center storage},
	author={Xiang, Yu and Lan, Tian and Aggarwal, Vaneet and Chen, Yih-Farn R},
	journal={IEEE/ACM Transactions on Networking (TON)},
	volume={24},
	number={4},
	pages={2443--2457},
	year={2016},
	publisher={IEEE Press}
}

@article{auer2002finite,
  title={Finite-time analysis of the multiarmed bandit problem},
  author={Auer, Peter and Cesa-Bianchi, Nicolo and Fischer, Paul},
  journal={Machine Learning},
  volume={47},
  number={2-3},
  pages={235--256},
  year={2002},
  publisher={Springer}
}

@article{thompson1933likelihood,
  title={On the likelihood that one unknown probability exceeds another in view of the evidence of two samples},
  author={Thompson, William R},
  journal={Biometrika},
  volume={25},
  number={3/4},
  pages={285--294},
  year={1933},
  publisher={JSTOR}
}

@incollection{NIPS2007_3371,
title = {The Price of Bandit Information for Online Optimization},
author = {Varsha Dani and Kakade, Sham M and Thomas P. Hayes},
booktitle = {Advances in Neural Information Processing Systems 20},
editor = {J. C. Platt and D. Koller and Y. Singer and S. T. Roweis},
pages = {345--352},
year = {2008},
publisher = {Curran Associates, Inc.},
}

@article{Cesa-Bianchi:2012:CB:2240304.2240495,
 author = {Cesa-Bianchi, Nicol\`{o} and Lugosi, G\'{a}bor},
 title = {Combinatorial Bandits},
 journal = {J. Comput. Syst. Sci.},
 issue_date = {September, 2012},
 volume = {78},
 number = {5},
 month = sep,
 year = {2012},
 issn = {0022-0000},
 pages = {1404--1422},
 numpages = {19},
 url = {http://dx.doi.org/10.1016/j.jcss.2012.01.001},
 doi = {10.1016/j.jcss.2012.01.001},
 acmid = {2240495},
 publisher = {Academic Press, Inc.},
 address = {Orlando, FL, USA},
}

@article{Audibert:2014:ROC:2765232.2765234,
 author = {Audibert, Jean-Yves and Bubeck, S{\'e}bastien and Lugosi, G\'{a}bor},
 title = {Regret in Online Combinatorial Optimization},
 journal = {Math. Oper. Res.},
 issue_date = {February 2014},
 volume = {39},
 number = {1},
 month = feb,
 year = {2014},
 issn = {0364-765X},
 pages = {31--45},
 numpages = {15},
 url = {http://dx.doi.org/10.1287/moor.2013.0598},
 doi = {10.1287/moor.2013.0598},
 acmid = {2765234},
 publisher = {INFORMS},
 address = {Institute for Operations Research and the Management Sciences (INFORMS), Linthicum, Maryland, USA},
}

@InProceedings{pmlr-v84-liau18a,
  title = 	 {Stochastic Multi-armed Bandits in Constant Space},
  author = 	 {David Liau and Zhao Song and Eric Price and Ger Yang},
  booktitle = 	 {Proceedings of the Twenty-First International Conference on Artificial Intelligence and Statistics},
  pages = 	 {386--394},
  year = 	 {2018},
  editor = 	 {Amos Storkey and Fernando Perez-Cruz},
  volume = 	 {84},
  series = 	 {Proceedings of Machine Learning Research},
  address = 	 {Playa Blanca, Lanzarote, Canary Islands},
  month = 	 {09--11 Apr},
  publisher = 	 {PMLR},
  url = 	 {http://proceedings.mlr.press/v84/liau18a.html}
}

@article{kveton2014matroid,
  title={Matroid bandits: Fast combinatorial optimization with learning},
  author={Kveton, Branislav and Wen, Zheng and Ashkan, Azin and Eydgahi, Hoda and Eriksson, Brian},
  journal={arXiv preprint arXiv:1403.5045},
  year={2014}
}

@inproceedings{gai2010learning,
  title={Learning multiuser channel allocations in cognitive radio networks: A combinatorial multi-armed bandit formulation},
  author={Gai, Yi and Krishnamachari, Bhaskar and Jain, Rahul},
  booktitle={New Frontiers in Dynamic Spectrum, 2010 IEEE Symposium on},
  pages={1--9},
  year={2010},
  organization={IEEE}
}

@InProceedings{pmlr-v28-chen13a,
  title = 	 {Combinatorial Multi-Armed Bandit: General Framework and Applications},
  author = 	 {Wei Chen and Yajun Wang and Yang Yuan},
  booktitle = 	 {Proceedings of the 30th International Conference on Machine Learning},
  pages = 	 {151--159},
  year = 	 {2013},
  editor = 	 {Sanjoy Dasgupta and David McAllester},
  volume = 	 {28},
  series = 	 {Proceedings of Machine Learning Research},
  address = 	 {Atlanta, Georgia, USA},
  month = 	 {17--19 Jun},
  publisher = 	 {PMLR},
}

@InProceedings{pmlr-v32-lind14,
  title = 	 {Combinatorial Partial Monitoring Game with Linear Feedback and Its Applications},
  author = 	 {Tian Lin and Bruno Abrahao and Robert Kleinberg and John Lui and Wei Chen},
  booktitle = 	 {Proceedings of the 31st International Conference on Machine Learning},
  pages = 	 {901--909},
  year = 	 {2014},
  editor = 	 {Eric P. Xing and Tony Jebara},
  volume = 	 {32},
  series = 	 {Proceedings of Machine Learning Research},
  address = 	 {Bejing, China},
  month = 	 {22--24 Jun},
  publisher = 	 {PMLR},
}

@inproceedings{li2010contextual,
  title={A contextual-bandit approach to personalized news article recommendation},
  author={Li, Lihong and Chu, Wei and Langford, John and Schapire, Robert E},
  booktitle={Proceedings of the 19th international conference on World wide web},
  pages={661--670},
  year={2010}
}

@incollection{hoeffding1994probability,
  title={Probability inequalities for sums of bounded random variables},
  author={Hoeffding, Wassily},
  booktitle={The Collected Works of Wassily Hoeffding},
  pages={409--426},
  year={1994},
  publisher={Springer}
}

@inproceedings{lattimore2018toprank,
  title={Toprank: A practical algorithm for online stochastic ranking},
  author={Lattimore, Tor and Kveton, Branislav and Li, Shuai and Szepesvari, Csaba},
  booktitle={Advances in Neural Information Processing Systems},
  pages={3945--3954},
  year={2018}
}

@article{gai2012combinatorial,
  title={Combinatorial network optimization with unknown variables: Multi-armed bandits with linear rewards and individual observations},
  author={Gai, Yi and Krishnamachari, Bhaskar and Jain, Rahul},
  journal={IEEE/ACM Transactions on Networking},
  volume={20},
  number={5},
  pages={1466--1478},
  year={2012},
  publisher={IEEE}
}

@inproceedings{kalyanakrishnan2012pac,
  author    = {Shivaram Kalyanakrishnan and
               Ambuj Tewari and
               Peter Auer and
               Peter Stone},
  title     = {{PAC} Subset Selection in Stochastic Multi-armed Bandits},
  booktitle = {Proceedings of the 29th International Conference on Machine Learning,
               {ICML} 2012, Edinburgh, Scotland, UK, June 26 - July 1, 2012},
  publisher = {icml.cc / Omnipress},
  year      = {2012},
  url       = {http://icml.cc/2012/papers/359.pdf},
  bibsource = {dblp computer science bibliography, https://dblp.org}
}

@inproceedings{kveton2015cascading,
  title={Cascading bandits: Learning to rank in the cascade model},
  author={Kveton, Branislav and Szepesvari, Csaba and Wen, Zheng and Ashkan, Azin},
  booktitle={International Conference on Machine Learning},
  pages={767--776},
  year={2015}
}

@book{bercu2015concentration,
  title={Concentration inequalities for sums and martingales},
  author={Bercu, Bernard and Delyon, Bernard and Rio, Emmanuel},
  year={2015},
  publisher={Springer}
}

@inproceedings{cohen2017tight,
  title={Tight Bounds for Bandit Combinatorial Optimization},
  author={Cohen, Alon and Hazan, Tamir and Koren, Tomer},
  booktitle={Conference on Learning Theory},
  pages={629--642},
  year={2017}
}

@article{audibert2014regret,
  title={Regret in Online Combinatorial Optimization},
  author={Audibert, Jean-Yves and Bubeck, S{\'e}bastien and Lugosi, G{\'a}bor},
  journal={Mathematics of Operations Research},
  volume={39},
  number={1},
  pages={31--45},
  year={2014},
  publisher={INFORMS}
}

@inproceedings{merlis2019batch,
  title={Batch-size independent regret bounds for the combinatorial multi-armed bandit problem},
  author={Merlis, Nadav and Mannor, Shie},
  booktitle={Conference on Learning Theory},
  pages={2465--2489},
  year={2019},
  organization={PMLR}
}

@inproceedings{merlis2020tight,
  title={Tight lower bounds for combinatorial multi-armed bandits},
  author={Merlis, Nadav and Mannor, Shie},
  booktitle={Conference on Learning Theory},
  pages={2830--2857},
  year={2020},
  organization={PMLR}
}

@inproceedings{agrawal2013thompson,
  title={Thompson sampling for contextual bandits with linear payoffs},
  author={Agrawal, Shipra and Goyal, Navin},
  booktitle={International conference on machine learning},
  pages={127--135},
  year={2013},
  organization={PMLR}
}

@article{agarwal2022stochastic,
  title={Stochastic Top K-Subset Bandits with Linear Space and Non-Linear Feedback with Applications to Social Influence Maximization},
  author={Agarwal, Mridul and Aggarwal, Vaneet and Umrawal, Abhishek K and Quinn, Christopher J},
  journal={ACM/IMS Transactions on Data Science (TDS)},
  volume={2},
  number={4},
  pages={1--39},
  year={2022},
  publisher={ACM New York, NY}
}

@inproceedings{dimakopoulou2019marginal,
  title={Marginal Posterior Sampling for Slate Bandits.},
  author={Dimakopoulou, Maria and Vlassis, Nikos and Jebara, Tony},
  booktitle={IJCAI},
  pages={2223--2229},
  year={2019}
}

@article{liu2021map,
  title={A map of bandits for e-commerce},
  author={Liu, Yi and Li, Lihong},
  journal={arXiv preprint arXiv:2107.00680},
  year={2021}
}

\end{document}